\documentclass{article}

\PassOptionsToPackage{hidelinks}{hyperref}
\usepackage[utf8]{inputenc}
\usepackage[T1]{fontenc}
\usepackage{xspace}
\usepackage[accepted]{icml2026}
\usepackage{hyperref}

\usepackage{times}
\usepackage{soul}
\usepackage[utf8]{inputenc}
\usepackage{caption}
\usepackage{graphicx}
\usepackage{microtype}
\usepackage{lineno}
\usepackage{lipsum,booktabs}
\usepackage{amsmath,mathrsfs,amssymb,amsfonts,bm}
\usepackage{rotating}
\usepackage{pdflscape}
\usepackage{tcolorbox}
\usepackage{wrapfig}

\usepackage[hidelinks]{hyperref}
\usepackage{url,dsfont,nicefrac}
\hypersetup{
    colorlinks,
    breaklinks,
    linkcolor = blue,
    citecolor = blue,
    urlcolor  = black,
}
\usepackage{multirow}
\usepackage{makecell}
\usepackage{anyfontsize}
\usepackage{bbding}
\usepackage{paralist}
\usepackage{enumerate}
\allowdisplaybreaks
\usepackage{appendix}
\usepackage{multirow,makecell,tabularx}
\usepackage{threeparttable}
\usepackage{pifont}
\usepackage{tikz}
\usepackage{wrapfig}

\usepackage[ruled,linesnumbered,algo2e]{algorithm2e}
\usepackage{algorithmic,algorithm}
\renewcommand{\algorithmicrequire}{ \textbf{Input:}}
\renewcommand{\algorithmicensure}{ \textbf{Output:}}
\newlength\myindent
\renewcommand{\hat}{\widehat}

\def \D {\mathcal{D}}
\def \E {\mathbb{E}}

\def \R {\mathbb{R}}

\def \X {\mathcal{X}}

\def \P {\mathbb{P}}

\def \thetah {\hat{\theta}}

\usepackage{mathtools}
\usepackage{autobreak}
\let\norm\undefined 
\newcommand\norm[1]{\left\| #1 \right\|}

\newcommand\abs[1]{\left| #1 \right|}

\newcommand\ceil[1]{\lceil #1 \rceil}
\newcommand*\diff{\mathop{}\!\mathrm{d}}

\newcommand\ind[1]{\mathds{1}_{\{#1\}}}
\newcommand\sbr[1]{\left( #1 \right)}
\newcommand\mbr[1]{\left[ #1 \right]}
\newcommand\bbr[1]{\left\{ #1 \right\}}

\newcommand\given[1][]{\:#1\vert\:}

\DeclareMathOperator*{\argmin}{arg\,min}

\def \define {\triangleq}

\numberwithin{equation}{section}

\usepackage{amsthm}

\newtheorem{myThm}{Theorem}

\newtheorem{myLemma}{Lemma}

\theoremstyle{definition}
\newtheorem{myAssum}{Assumption}

\usepackage{subcaption}
\usepackage{color}

\definecolor{wine_red}{RGB}{228,48,64}
\definecolor{DSgray}{cmyk}{0,1,0,0}

\usepackage{prettyref}

\newcommand{\savehyperref}[2]{\texorpdfstring{\hyperref[#1]{#2}}{#2}}
\newrefformat{eq}{\savehyperref{#1}{Eq.~\textup{(\ref*{#1})}}}
\newrefformat{eqn}{\savehyperref{#1}{Eq.~(\ref*{#1})}}
\newrefformat{lem}{\savehyperref{#1}{Lemma~\ref*{#1}}}
\newrefformat{lemma}{\savehyperref{#1}{Lemma~\ref*{#1}}}
\newrefformat{def}{\savehyperref{#1}{Definition~\ref*{#1}}}
\newrefformat{line}{\savehyperref{#1}{Line~\ref*{#1}}}
\newrefformat{thm}{\savehyperref{#1}{Theorem~\ref*{#1}}}
\newrefformat{table}{\savehyperref{#1}{Table~\ref*{#1}}}
\newrefformat{corr}{\savehyperref{#1}{Corollary~\ref*{#1}}}
\newrefformat{cor}{\savehyperref{#1}{Corollary~\ref*{#1}}}
\newrefformat{sec}{\savehyperref{#1}{Section~\ref*{#1}}}
\newrefformat{subsec}{\savehyperref{#1}{Section~\ref*{#1}}}
\newrefformat{app}{\savehyperref{#1}{Appendix~\ref*{#1}}}
\newrefformat{appendix}{\savehyperref{#1}{Appendix~\ref*{#1}}}
\newrefformat{assum}{\savehyperref{#1}{Assumption~\ref*{#1}}}
\newrefformat{ex}{\savehyperref{#1}{Example~\ref*{#1}}}
\newrefformat{fig}{\savehyperref{#1}{Figure~\ref*{#1}}}
\newrefformat{alg}{\savehyperref{#1}{Algorithm~\ref*{#1}}}
\newrefformat{remark}{\savehyperref{#1}{Remark~\ref*{#1}}}
\newrefformat{conj}{\savehyperref{#1}{Conjecture~\ref*{#1}}}
\newrefformat{prop}{\savehyperref{#1}{Proposition~\ref*{#1}}}
\newrefformat{proto}{\savehyperref{#1}{Protocol~\ref*{#1}}}
\newrefformat{prob}{\savehyperref{#1}{Problem~\ref*{#1}}}
\newrefformat{property}{\savehyperref{#1}{Property~\ref*{#1}}}
\newrefformat{claim}{\savehyperref{#1}{Claim~\ref*{#1}}}
\newrefformat{que}{\savehyperref{#1}{Question~\ref*{#1}}}
\newrefformat{op}{\savehyperref{#1}{Open Problem~\ref*{#1}}}
\newrefformat{fn}{\savehyperref{#1}{Footnote~\ref*{#1}}}
\newrefformat{require}{\savehyperref{#1}{Requirement~\ref*{#1}}}

\icmltitlerunning{Robust Length Prediction: A Perspective from Heavy-Tailed Prompt-Conditioned Distributions}

\begin{document}

\twocolumn[
\icmltitle{Robust Length Prediction: A Perspective from Heavy-Tailed Prompt-Conditioned Distributions}

\begin{icmlauthorlist}
    \icmlauthor{Jing Wang}{keylab,ai}
    \icmlauthor{Yu-Yang Qian}{keylab,ai}
    \icmlauthor{Ke Xue}{keylab,ai}
    \icmlauthor{Chao Qian}{keylab,ai}
    \icmlauthor{Peng Zhao}{keylab,ai}
    \icmlauthor{Zhi-Hua Zhou}{keylab,ai}
\end{icmlauthorlist}

\icmlaffiliation{keylab}{State Key Laboratory of Novel Software Technology, Nanjing University, Nanjing 210023, China}
\icmlaffiliation{ai}{School of Artificial Intelligence, Nanjing University, Nanjing 210023, China}

\icmlcorrespondingauthor{Zhi-Hua Zhou}{zhouzh@lamda.nju.edu.cn}
\icmlkeywords{Large Language Models, Length Prediction, Heavy Tail Distribution, Output Length Inference}

\vskip 0.3in
]

\printAffiliationsAndNotice{}

\begin{abstract}
    Output-length prediction is important for efficient LLM serving, as it directly affects batching, memory reservation, and scheduling. For prompt-only length prediction, most existing methods use a one-shot sampled length as the label, implicitly treating each prompt as if it had one true target length. We show that this is unreliable: even under a fixed model and decoding setup, the same prompt induces a \emph{prompt-conditioned output length distribution}, not a deterministic scalar, and this distribution is consistent with \emph{heavy-tailed} behavior. Motivated by this, we cast length prediction as robust estimation from heavy-tailed prompt-conditioned length distributions. We propose prompt-conditioned length distribution (ProD) methods, which construct training targets from multiple independent generations of the same prompt. Two variants are developed to reuse the served LLM's hidden states: \mbox{ProD-M}, which uses a median-based target for robust point prediction, and ProD-D, which uses a distributional target that preserves prompt-conditioned uncertainty. We provide theoretical justifications by analyzing the estimation error under a surrogate model. Experiments across diverse scenarios show consistent gains in prediction quality.
\end{abstract}

\section{Introduction}
\label{sec:introduction}

Large language models (LLMs) have achieved remarkable success across a wide range of tasks, including chatbots~\citep{NeurIPS'22:ChatGPT}, mathematical reasoning~\citep{Nature'25:DeepSeek-R1}, and agentic tasks~\citep{SCIS'25:LLMAgentSurvey}. As the inference cost of LLMs grows with increasing parameter size, the \emph{generation length} associated with different prompts emerges as a key factor for efficiency. For example, in serving scenarios~\citep{NeurIPS'23:sglang}, knowing the output length in advance enables better KV cache reservation and improved scheduling, directly enhancing throughput and reducing latency. In agentic workflows~\citep{LM'24:agent-network} where each action triggers one or more LLM calls, predicting the output length allows the orchestrator to improve end-to-end latency. These scenarios make \emph{length prediction} a key building block for efficient LLM inference.

From a broader perspective, modern LLM serving is a \emph{learning-augmented scheduling} problem: the system must allocate time-shared computational resources across concurrent requests, including GPU memory, compute slots, batch positions, etc. The quality of these allocation decisions hinges on learned predictions about future resource demands. This viewpoint is precisely captured by the CoRE-learning framework~\citep{NSR'24:core}, which argues that when computational resources are time-shared, learning and scheduling become inseparable; the learning component must produce signals that directly inform scheduling decisions, not merely optimize standalone prediction accuracy. Output length prediction is a concrete and central instantiation of this paradigm: it provides the serving scheduler with a per-request estimate of computational cost, bridging the gap between a learned predictor and the resource allocation it is meant to support.

Given this trend, \emph{length prediction} for LLMs has attracted growing attention. Previous works have devoted considerable effort to predicting output length. $S^3$~\citep{NeurIPS'23:s3} trains an auxiliary model to predict the output sequence length before generation. LTR~\citep{NeurIPS'24:ltr} employs a learning-to-rank model to predict the relative order of output lengths within a batch. TRAIL~\citep{ICLR'25:embedding-scheduling} attaches a lightweight classifier to the target LLM's intermediate embeddings to predict output length in an online manner, and EGTP~\citep{ICLR'26:EGTP} further improves representation quality through entropy-weighted informative token representations for more accurate length prediction.

\begin{figure*}[t]
    \centering
    \begin{subfigure}[t]{0.32\linewidth}
        \centering
        \includegraphics[width=\linewidth]{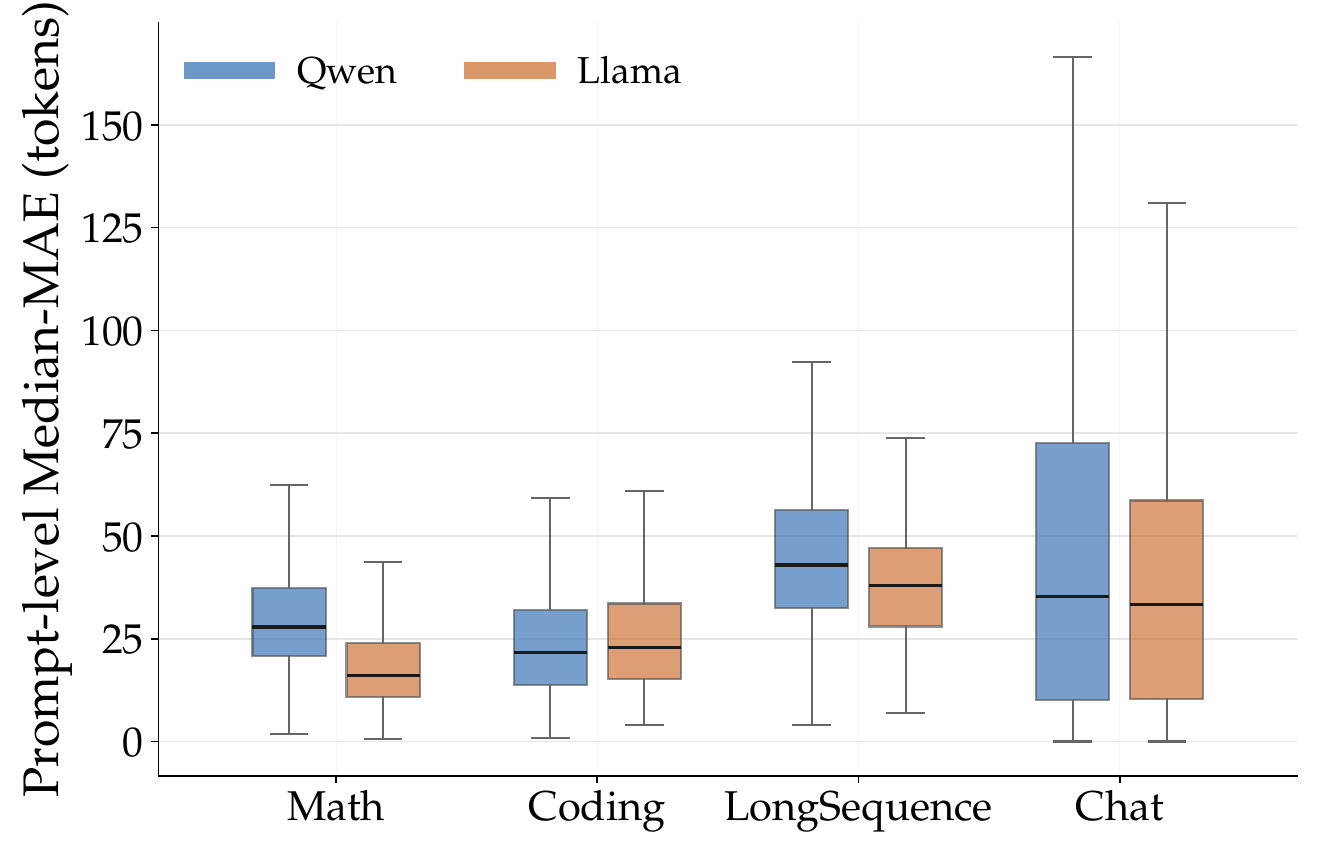}
        \caption{Noise radius across 8 settings.}
        \label{fig:key_observations_noise_floor}
    \end{subfigure}
    \hfill
    \begin{subfigure}[t]{0.32\linewidth}
        \centering
        \includegraphics[width=\linewidth]{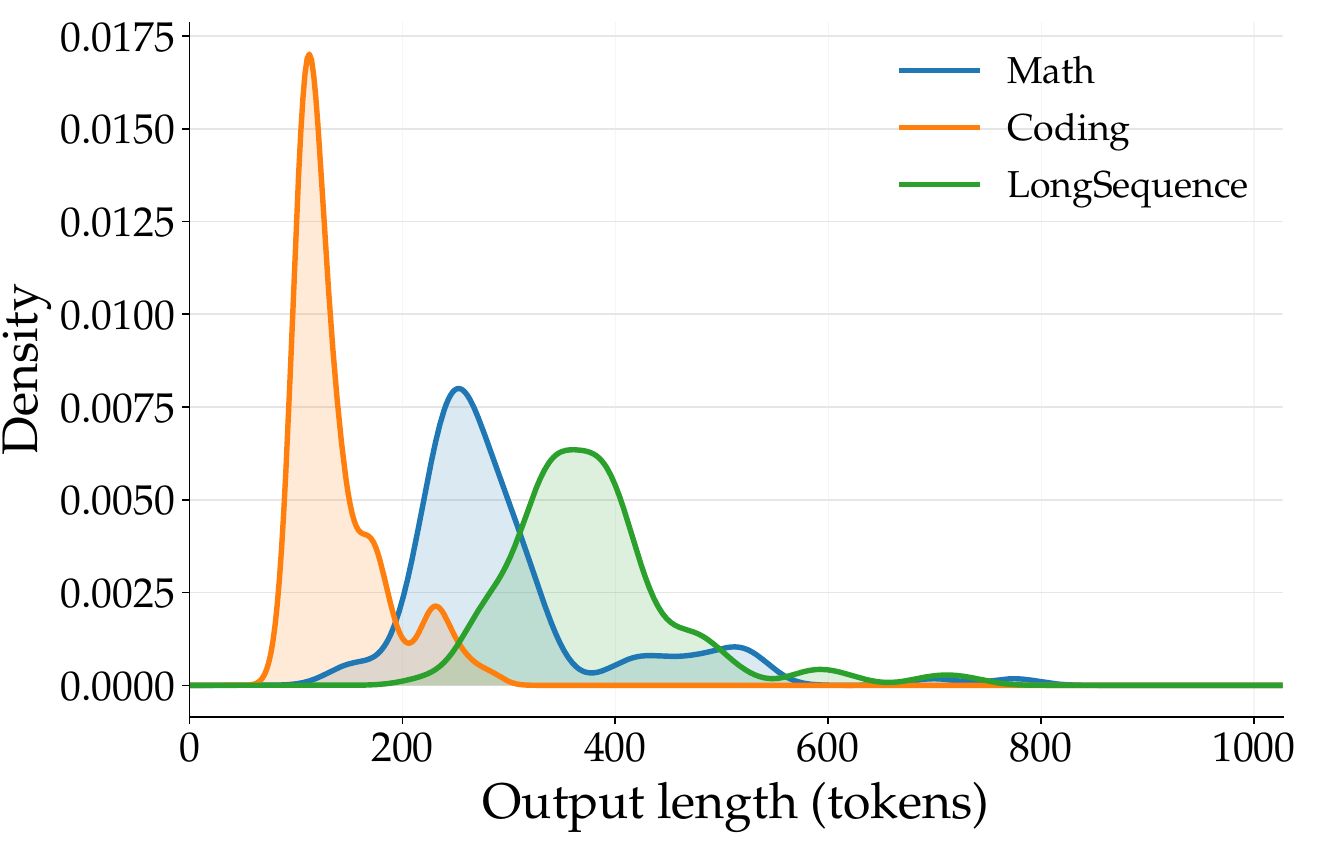}
        \caption{Heavy-tail examples on Qwen.}
        \label{fig:key_observations_heavy_tail_qwen}
    \end{subfigure}
    \hfill
    \begin{subfigure}[t]{0.32\linewidth}
        \centering
        \includegraphics[width=\linewidth]{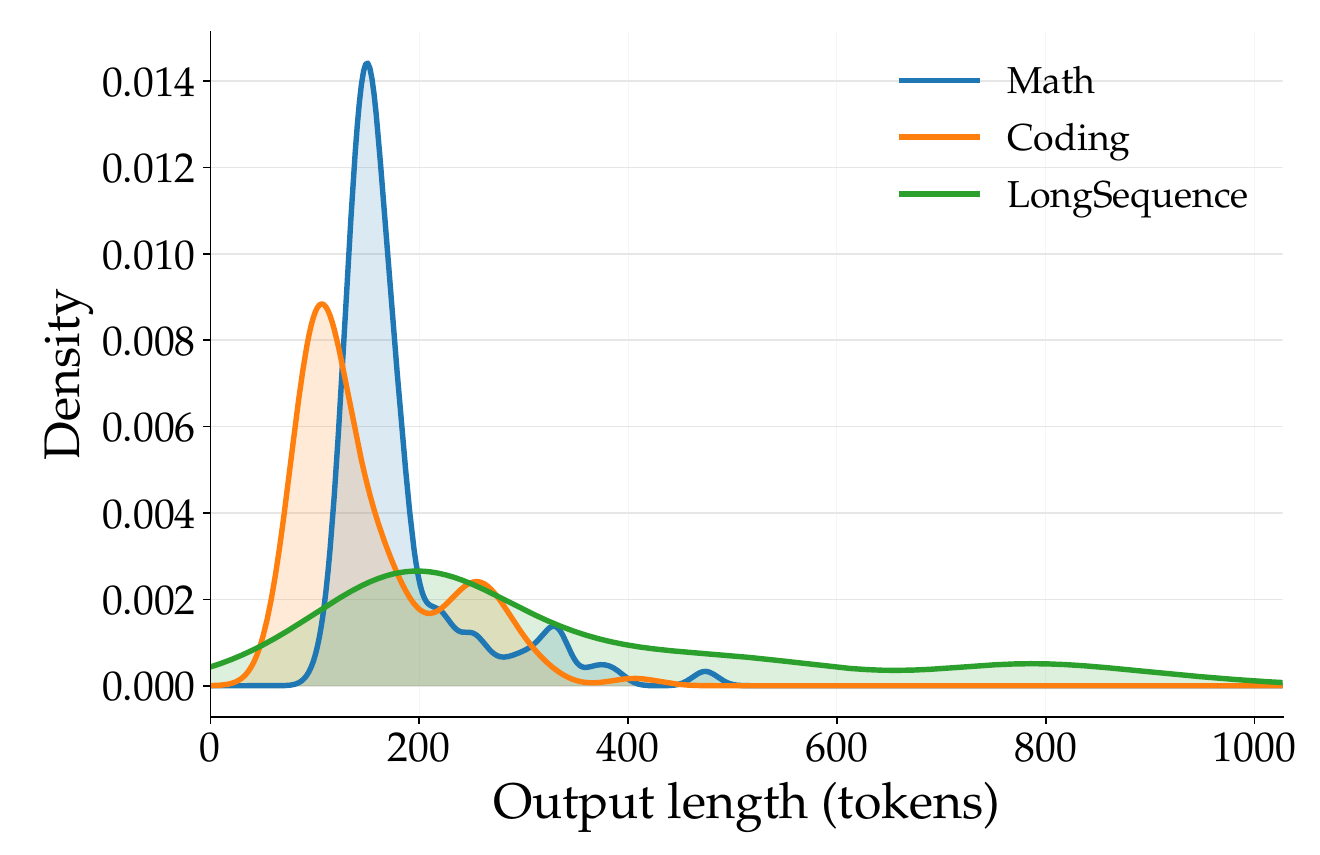}
        \caption{Heavy-tail examples on Llama.}
        \label{fig:key_observations_heavy_tail_llama}
    \end{subfigure}
    \caption{\textbf{Key observations about prompt-conditioned output length.} Figure (a) summarizes prompt-level median-centered noise radius across the Math, Coding, LongSequence, and Chat scenarios via repeated-sampling Median-MAE. Figures (b) and (c) show representative repeated-sampling length distributions for Math, Coding, and LongSequence prompts under Qwen and Llama. Additional per-setting supporting plots are provided in Appendix~\ref{sec:app_key_observations}.}
    \label{fig:key_observations}
\end{figure*}

Despite achieving remarkable empirical success, these methods typically treat output length prediction as a deterministic scalar, i.e., use a one-shot sampled length as the label. However, for LLM decoding, generations from the same prompt and fixed configuration can produce noticeably different output lengths \emph{in the form of a distribution}, revealing an inherent variance. Moreover, the prompt-conditioned output length distribution is often consistent with a clear \emph{heavy-tail} tendency, where occasional long generations shift the sample mean far from the stable central value. Consequently, training a predictor with a one-shot sampled length conflicts with the \emph{distributional nature} of the LLM's output length, yielding a statistically misaligned objective.

In this work, we reveal a fundamental yet overlooked fact: \emph{output length is drawn from a prompt-conditioned distribution rather than being a deterministic scalar}. As shown in Figure~\ref{fig:key_observations}, under a fixed model and decoding configuration, the same prompt can yield completions whose lengths vary by multiples of each other. This fundamentally changes the statistical nature of the prediction problem: the prompt-conditioned length distribution is consistent with heavy-tailed behavior, so occasional extremely long generations can shift the sample mean far from the stable central tendency, making single-sample supervision an unreliable and statistically misaligned training signal.

To tackle this problem, we reframe length prediction through the lens of \emph{Prompt-conditioned length Distributions} (\textbf{ProD}). We show that under the MAE criterion, the Bayes-optimal prediction is the conditional median, and propose two complementary supervision strategies built on repeated sampling. \emph{(i)} \textbf{ProD-M}: \emph{Median supervision} draws multiple independent generations per prompt and uses their median as the training label, compressing the heavy-tailed distribution into a single robust point target. \emph{(ii)} \textbf{ProD-D}: \emph{Distributional supervision} projects the sampled lengths onto a binned histogram and uses it as a soft training target, requiring the predictor to match the full length distribution. Both variants reuse the served LLM's last-layer hidden states as input representations, requiring no auxiliary model or additional inference cost. Theoretical analysis indicates that increasing the repeated-sampling budget exponentially reduces the failure probability of the estimation bound, justifying the soundness of our method. Experiments on Qwen-2.5-7B and Llama-3-8B across four benchmarks show that our method reduces average prediction error by up to $25\%$ over previous SOTA methods.

\paragraph{Organization.} Section~\ref{sec:approach} introduces our approach. Section~\ref{sec:experiments} empirically evaluates the proposed method.  Section~\ref{sec:related_work} reviews the related work and provides further discussion.  Finally, Section~\ref{sec:conclusion} concludes the paper.

\section{Proposed Approach}
\label{sec:approach}
In this section, we first identify two key observations in the length-prediction problem. We then present the problem formulation and develop two components for robust length prediction: \textit{(i)} estimation by repeated sampling and \textit{(ii)} prediction by robust statistics.

\subsection{Key Observations}
Output-length prediction under a fixed served model and decoding configuration is fundamentally a stochastic conditional inference problem, rather than deterministic regression to one length label. To make this concrete, we study Qwen-2.5-7B and Llama-3-8B across four scenarios: Math, Coding, LongSequence, and Chat. Figure~\ref{fig:key_observations} highlights two recurring facts. First, once the served model and decoding setup are fixed within a setting, repeated generations of the same prompt still exhibit a median-centered noise radius. Second, even for a fixed prompt, the induced length distribution often has a clear heavy tail rather than a single deterministic target. Appendix~\ref{sec:app_key_observations} gives the full setting details, sample counts, and caveats. In the following, we discuss the two key observations.

\paragraph{Prompt-conditioned length has noise radius.}Figure~\ref{fig:key_observations_noise_floor} reports prompt-level Median-MAE from $16$ repeated generations per prompt, aggregated over all eight model--scenario settings. The setting-wise medians are already tens of tokens in every case: for Math they are $27.8$ and $16.1$ tokens on Qwen and Llama, for Coding $21.7$ and $23.0$, for LongSequence $42.9$ and $38.0$, and for Chat $35.3$ and $33.4$. Thus, the same prompt induces a substantial stochastic spread in output length. After normalizing by the prompt-wise median length, the median noise ratio still remains non-trivial, from $11.5\%$ on Qwen/Math to $18.2\%$ on Llama/LongSequence. Math is the most stable regime in this summary, whereas LongSequence and Chat carry several of the largest medians and upper tails. The prompt-only predictor is therefore not recovering a deterministic label, but estimating a \emph{prompt-conditioned output length distribution}.

\paragraph{Prompt-conditioned length is heavy-tailed.} Figures~\ref{fig:key_observations_heavy_tail_qwen} and~\ref{fig:key_observations_heavy_tail_llama} show six frozen representative prompts from the Math, Coding, and LongSequence scenarios. Their max-to-median ratios are $2.91\times$, $2.42\times$, and $1.98\times$ for Qwen, and $2.99\times$, $3.27\times$, and $4.03\times$ for Llama. These prompts, therefore, consistent with a clear heavy-tail tendency rather than tight concentration around one typical length. Appendix~\ref{sec:app_key_observations} keeps these caveats explicit and extends the same light/heavy diagnostic to all eight settings, including the less regular Chat regime. This matters for supervision design: when the tail is occasionally long, a small number of unusually long generations can move the conditional mean substantially, whereas the conditional median remains much more stable.

These two observations motivate us to formulate the prompt-conditioned length prediction problem below. The first observation rules out treating output length as a deterministic regression target, so the predictor should target a central functional of the conditional distribution. The second observation further suggests that this functional should be robust to occasional long generations, which makes conditional-median estimation the natural starting point under MAE.

\subsection{Formulation: Prompt-Conditioned Distribution}
\label{sec:problem_formulation}
We investigate conditional length prediction for autoregressive large language models under a fixed served model and decoding configuration. For the $i$-th request, let $x_i \in \X$ denote the input prompt. During decoding, after revealing the first $t$ output tokens $\{y_i^s\}_{s=1}^t$, the currently observed context is $z_i^t \define \bbr{x_i, \{y_i^s\}_{s=1}^t},~ z_i^0 = x_i$. Let $\phi(z_i^t)\in\R^d$ denote the representation of the available history at state $z_i^t$. Starting from state $z_i^t$ and continuing decoding until EOS, let $L_i^t$ denote the remaining number of tokens to be generated. During LLM decoding, $L_i^t$ is generally random even for the same state $z_i^t$. We assume that there exists a representation $\phi(\cdot)$ such that $\phi(z_i^t)$ is a sufficient statistic for length prediction, in the sense that $\forall i, t$, the length distribution is conditional on $\phi(z_i^t)$, i.e., $P(L_i^t \mid \phi(z_i^t))$. The prompt-only setting is the special case $t=0$, for which $z_i^0= x_i$ and $L_i \define L_i^0$.

We consider the static case where the predictor only observes the prompt representation $\phi(x_i)$ and outputs a scalar estimate $\hat L_i$ of the total output length. We evaluate such predictors using conditional mean absolute error (MAE):
\begin{equation}
\mathcal{R}_{\mathrm{pt}}(\hat L_i \mid \phi(x_i))
\define
\E_{L_i \sim p^\star(\cdot \mid \phi(x_i))}
\mbr{|L_i-\hat L_i|}.
\label{eq:point_risk}
\end{equation}
This objective directly matches the first observation: the predictor must summarize a stochastic prompt-conditioned distribution rather than recover a single ground-truth label. It also aligns with the second observation: when the conditional distribution can be heavy-tailed, a robust central estimate is operationally more stable than a supervision target that is overly sensitive to rare long generations.

\subsection{Estimation by Repeated Sampling}

We now instantiate the prompt-only estimator motivated above. The logic has four steps. Under MAE, the correct population target is a conditional median. A single realized output length is only one stochastic draw from that target distribution. Repeated sampling lets us replace that noisy single draw with a more stable sample median. We then use a simplified linear surrogate to isolate why this supervision choice improves estimation reliability.

We begin with the population target. For any integer-valued candidate prediction $a$, define the conditional point risk as $\mathcal{R}_{\mathrm{pt}}(a \given \phi(x_i))\define \E\mbr{\abs{L_i-a}\given \phi(x_i)}$. Since $L_i$ is integer-valued, a direct discrete-difference calculation gives
\begin{equation*}
    \mathcal{R}_{\mathrm{pt}}(a+1 \given \phi(x_i))
    -
    \mathcal{R}_{\mathrm{pt}}(a \given \phi(x_i))
    =
    2\,\P\mbr{L_i \le a \given \phi(x_i)} - 1.
\end{equation*}
Hence, $\mathcal{R}_{\mathrm{pt}}(a \given \phi(x_i))$ decreases as long as $\P(L_i \le a \given \phi(x_i)) < 1/2$ and increases once $\P(L_i \le a \given \phi(x_i)) > 1/2$. Therefore, among all predictors measurable with respect to $\phi(x_i)$, the Bayes-optimal point target under MAE is a conditional median of $L_i$ given $\phi(x_i)$.

The main supervision difficulty is that a single realized output length is only one stochastic sample from this conditional distribution. Using that single sample directly as a training label injects sampling randomness into supervision, exactly the noise source highlighted by Observation~1. To better align the supervision signal with the conditional-median target above, we draw $r$ independent generations for each prompt $x_i$, obtain the corresponding output lengths $\{L_{i,1},\dots,L_{i,r}\}$, and use their sample median $\bar{L}_i \define \operatorname{median}\mbr{L_{i,1},\dots,L_{i,r}}$ as the supervision label. This design is also consistent with Observation~2: when prompt-conditioned length distributions can be consistent with heavy-tailed behavior, the sample median is a more robust empirical target than either a single sampled length or an average dominated by a few unusually long generations. Repeated sampling is used only to construct robust labels during training; it is not part of the predictor input or inference procedure at test time.

To isolate the statistical value of repeated-sampling median supervision, we analyze a simplified linear surrogate estimator. Specifically, we model the output length as $L_i = \phi(x_i)^\top \theta_* + \eta_i$, where $\theta_* \in \R^d$ is an unknown parameter satisfying $\norm{\theta_*}_2 \le S$. We also assume bounded representations, i.e., $\norm{\phi(x_i)}_2 \le 1$ for all $x_i \in \D$. This linear model is not intended to exactly match the deployed nonlinear predictor; rather, it serves as a tractable analytical surrogate for separating the supervision effect of repeated sampling from the full modeling complexity of the deployed system. Motivated by Observation~2, we allow for noise that may not be strongly concentrated. The following assumption is a proof-convenient sufficient condition for the analysis, rather than a literal claim that every empirical distribution exactly follows this form:
\begin{myAssum}\label{assumption:noise_symmetric}
The noise $\eta$ is symmetric around zero and satisfies
\begin{equation*}
    \E\mbr{\abs{\eta}^{1+\epsilon} \mid x} \le v,
\end{equation*}
for some $\epsilon \in (0,1]$ and $v > 0$.
\end{myAssum}
Although our statistical target under MAE is the conditional median rather than the conditional mean, we retain symmetry as a proof-convenient sufficient condition: it implies a zero conditional median and enables robust control of the repeated-sampling median labels.

Given $N$ training prompts, we construct one median label $\bar{L}_i$ for each prompt and fit a ridge estimator:
\begin{equation*}
    \thetah_N
    =
    \argmin_{\theta}
    \lambda \norm{\theta}_2^2
    +
    \sum_{i=1}^{N}
    \sbr{\phi(x_i)^\top \theta - \bar{L}_i}^2,
\end{equation*}
where $\lambda > 0$ is the regularization coefficient. The following lemma quantifies the estimation effect of using repeated-sampling median labels in this surrogate setting.

\begin{myThm}[Expected absolute error bound]
    \label{theorem:estimation_error}
    For any $\delta \in (0,1)$, define $C \define (4v)^{\frac{1}{1+\epsilon}}, \rho_\delta \define 2C\ln\sbr{\frac{8N}{\delta}} + 4C^{-\epsilon}v$. Then, with probability at least $1-\delta-4Ne^{-r/8}$, for all $x \in \X$:
    \begin{equation*}
        \abs{\phi(x)^\top \theta_* - \phi(x)^\top \thetah_N}
        \le
        \beta_{N} \norm{\phi(x)}_{V_{N}^{-1}}.
    \end{equation*}
    where $\beta_N
    \define
    \sqrt{
        \rho_\delta^2 N^{\frac{1-\epsilon}{1+\epsilon}}
        +
        2C\rho_\delta d\, N^{\frac{1-\epsilon}{1+\epsilon}}
        \log\sbr{1+\frac{N}{\lambda d}}
    } +\sqrt{\lambda}S
    $. In particular, if $r \ge 8\log\sbr{4N/\delta}$, then the above bound holds with probability at least $1-2\delta$.
\end{myThm}
This Theorem clarifies the role of repeated sampling in our analysis from two perspectives. First, the failure probability contains the additional term $4N e^{-r/8}$, which decays exponentially with the repeated-sampling budget $r$. In particular, when $r \ge 8\log(4N/\delta)$, this term is absorbed into the main confidence level, and the bound holds with probability at least $1-2\delta$. Thus, increasing the repeated-sampling budget directly improves the reliability of the supervision labels and, in turn, the resulting estimation guarantee. Second, the prediction error scales with $\norm{\phi(x)}_{V_N^{-1}}$, the standard self-normalized uncertainty term. When the feature of a test prompt is better covered by the training feature covariance $V_N$, this quantity becomes smaller, leading to a tighter prediction error bound. The proof is provided in Appendix~\ref{sec:app_theory}.

\subsection{Prediction by Robust Statistics}
Since our goal is to isolate the value of repeated sampling itself, we keep the predictor architecture fixed and vary only how repeated observations are converted into training signals. Specifically, we let $\phi(x_i)$ be the last-layer hidden state of the last prompt token, following the lightweight probing style of TRAIL~\citep{ICLR'25:embedding-scheduling}. On top of this representation, we use the same two-layer MLP for both variants: the first layer maps $\phi(x_i)\in\R^d$ to a $512$-dimensional hidden vector with a ReLU nonlinearity, and the second layer outputs logits over a shared grid of $K$ length bins. Let $q_\theta(\cdot\mid x_i)\define \operatorname{softmax}(g_\theta(\phi(x_i))) \in \Delta^{K-1}$, where $g_\theta$ denotes this shared predictor.

\textbf{ProD-M.} For each training prompt $x_i$, we draw $r$ independent generations and obtain the corresponding output lengths $\{L_{i,1},\dots,L_{i,r}\}$. We then compress these repeated observations into a single robust point target via the sample median $\bar{L}_i=\operatorname{median}(L_{i,1},\dots,L_{i,r})$. After discretizing $\bar{L}_i$ onto the same $K$-bin grid, we obtain a one-hot target vector $y_i^{\mathrm{med}}\in\{0,1\}^K$ and train the classifier with standard cross-entropy:
\begin{equation*}
\mathcal{L}_{\mathrm{med}}(\theta)
\define
-\sum_{i}\sum_{k=1}^{K} y_i^{\mathrm{med}}(k)\log q_\theta(k\mid x_i).
\end{equation*}
This variant treats repeated sampling as a way to denoise supervision before learning: instead of fitting one stochastic realization, it learns from a prompt-level target that is better aligned with the conditional median under MAE.

\begin{table*}[t]
    \centering
    \caption{\textbf{Prompt-only length prediction results.} Test MAE (tokens) between each method's prediction and the prompt-level median output length under $16$-sample repeated sampling. Noise Radius is the prompt-level median-centered noise radius. Lower is better; the best trainable result per column is \textbf{bold}. The exact protocol is in Appendix~\ref{sec:app_benchmark_protocol}.}
    \label{tab:main_prompt_only_skeleton}

\begin{small}
    \setlength{\tabcolsep}{3pt}
    \renewcommand{\arraystretch}{1.04}
    \begin{tabular}{lcccccccccc}
        \toprule
        \multirow{2}{*}[-.5ex]{Method} & \multicolumn{5}{c}{Qwen} & \multicolumn{5}{c}{Llama} \\
        \cmidrule(lr){2-6}\cmidrule(l){7-11}
        & Math & Coding & LongSeq & Chat & Avg & Math & Coding & LongSeq & Chat & Avg \\
        \midrule
        Constant Median & 59.59 & 56.41 & 146.11 & 264.90 & 131.75 & 32.07 & 37.83 & 95.69 & 213.31 & 94.73 \\
        $S^3$ & 41.60 & 57.28 & 72.21 & 185.84 & 89.23 & 24.41 & 41.28 & 50.01 & 142.28 & 64.50 \\
        TRAIL-mean & 44.04 & 44.70 & 72.70 & 143.50 & 76.24 & 26.29 & 37.93 & 50.99 & 116.84 & 58.01 \\
        TRAIL-last & 35.13 & 36.63 & 57.68 & 127.47 & 64.23 & 24.51 & 35.74 & 41.04 & 95.03 & 49.08 \\
        EGTP & 49.50 & 57.14 & 69.84 & 276.27 & 113.19 & 28.25 & 37.98 & 46.73 & 148.08 & 65.26 \\
        \textbf{ProD-M} & 30.80 & 32.08 & 55.54 & 124.53 & 60.74 & 20.32 & 29.96 & 38.13 & 94.59 & 45.75 \\
        \textbf{ProD-D} & \textbf{30.35} & \textbf{31.76} & \textbf{51.41} & \textbf{113.60} & \textbf{56.78} & \textbf{19.57} & \textbf{26.61} & \textbf{37.68} & \textbf{93.39} & \textbf{44.31} \\
        Noise Radius & 32.94 & 25.18 & 57.18 & 56.05 & 42.84 & 20.39 & 27.07 & 56.38 & 45.71 & 37.38 \\
        \bottomrule
    \end{tabular}
\end{small}

\end{table*}

\textbf{ProD-D.} The second variant keeps the same representation, bin grid, and predictor head, but does not collapse the repeated observations into a single scalar. Instead, we project all repeated lengths $\{L_{i,1},\dots,L_{i,r}\}$ onto the same $K$ bins and form a prompt-level empirical distribution
\begin{equation*}
p_i^{\mathrm{dist}}(k)
\define
\frac{1}{r}\sum_{j=1}^{r}\mathbf{1}\!\left[b(L_{i,j})=k\right],
\end{equation*}
where $b(\cdot)$ maps a length to its bin index. We then use this empirical histogram as a soft target and train the same classifier by soft cross-entropy,
\begin{equation*}
\mathcal{L}_{\mathrm{dist}}(\theta)
\define
-\sum_{i}\sum_{k=1}^{K} p_i^{\mathrm{dist}}(k)\log q_\theta(k\mid x_i).
\end{equation*}
Compared with ProD-M, this variant preserves more of the prompt-conditioned uncertainty revealed by repeated sampling: it still uses repeated generations only at training time, but it asks the predictor to match the observed length distribution rather than only its center.

When a single scalar prediction is needed, both variants extract the median of the predictive distribution $q_\theta(\cdot\mid x_i)$: we compute the cumulative sum of the predicted bin probabilities, locate the bin where the cumulative mass crosses $0.5$, and linearly interpolate within that bin to obtain a continuous point estimate $\hat{L}_i$. Prior methods typically decode via the argmax bin center or the expectation of the bin distribution; using the median instead yields a more robust point estimate, as the median is less sensitive to heavy-tailed or skewed predicted distributions.

\section{Experiments}
\label{sec:experiments}

\subsection{Experimental Setup}
\paragraph{System and hardware.}
The evaluation is conducted on a server equipped with two Intel Xeon Platinum 8358 32-Core processors (turbo frequency up to 3.4 GHz), two NVIDIA Tesla A100 80GB GPUs, and 1024 GiB of RAM. The system runs Ubuntu 20.04 with CUDA 12.2, and uses vLLM~$0.15.1$~\citep{SOSP'23:pagedattention} as the serving framework.

\paragraph{Served models and benchmark scenarios.}
We use Qwen-2.5-7B~\citep{arXiv'24:qwen2.5} and Llama-3-8B~\citep{MetaAI:llama3} as the served models. The main prompt-only benchmark is organized around four anchor scenarios: Math (GSM8K~\citep{arXiv'21:training-verifiers-math}), Coding (MBPP~\citep{arXiv'21:program-synthesis-llm}), LongSequence (LongBench~\citep{ACL'24:longbench}), and Chat (LMSYS-Chat-1M~\citep{ICLR'24:lmsys-chat-1m}). We use the official train/test splits for GSM8K ($7473/1319$) and MBPP ($374/500$). For LongBench and LMSYS-Chat-1M, we derive deterministic train/test splits from the repeated-sampling source manifests. This yields LongBench splits of $3789/961$ on Qwen and $3780/970$ on Llama, and LMSYS-Chat-1M splits of $4070/930$ on both models. 

\begin{table*}[t]
    \centering
    \caption{\textbf{Single-sample supervision ablation (single-label evaluation).} All predictors are trained with one sampled output length per prompt and evaluated against the same single-label target. We report mean $\pm$ std test MAE over $8$ trials. Lower is better; the best per column is \textbf{bold}.}
    \label{tab:single_vs_single_main}

\begin{small}
    \setlength{\tabcolsep}{2.5pt}
    \renewcommand{\arraystretch}{1.04}
    \resizebox{\textwidth}{!}{%
    \begin{tabular}{lcccccccc}
        \toprule
        \multirow{2}{*}[-.5ex]{Method} & \multicolumn{4}{c}{Qwen} & \multicolumn{4}{c}{Llama} \\
        \cmidrule(lr){2-5}\cmidrule(l){6-9}
        & Math & Coding & LongSeq & Chat & Math & Coding & LongSeq & Chat \\
        \midrule
        $S^3$ & 54.79 $\pm$ 0.8 & 63.38 $\pm$ 1.1 & 96.00 $\pm$ 6.6 & 201.33 $\pm$ 4.7 & 33.31 $\pm$ 1.3 & 52.57 $\pm$ 2.9 & 82.97 $\pm$ 3.8 & 154.78 $\pm$ 4.1 \\
        \shortstack[l]{TRAIL-mean} & 59.74 $\pm$ 3.0 & 59.21 $\pm$ 2.8 & 99.58 $\pm$ 6.2 & 174.14 $\pm$ 5.6 & 34.59 $\pm$ 1.4 & 46.71 $\pm$ 3.1 & 87.13 $\pm$ 4.4 & 137.98 $\pm$ 6.5 \\
        \shortstack[l]{TRAIL-last} & 54.11 $\pm$ 8.2 & 56.44 $\pm$ 7.4 & 88.07 $\pm$ 4.9 & 154.45 $\pm$ 6.0 & 32.48 $\pm$ 1.1 & 45.29 $\pm$ 4.1 & 76.75 $\pm$ 2.9 & 112.84 $\pm$ 5.3 \\
        EGTP & 67.30 $\pm$ 7.5 & 62.89 $\pm$ 1.0 & 96.84 $\pm$ 6.1 & 278.12 $\pm$ 5.3 & 36.39 $\pm$ 1.6 & 49.55 $\pm$ 2.9 & 81.37 $\pm$ 3.6 & 163.95 $\pm$ 5.0 \\
        \textbf{ProD-M} & \textbf{47.10 $\pm$ 0.8} & \textbf{43.43 $\pm$ 1.1} & \textbf{86.00 $\pm$ 5.2} & \textbf{147.18 $\pm$ 4.6} & \textbf{29.32 $\pm$ 1.0} & \textbf{40.86 $\pm$ 3.8} & \textbf{73.22 $\pm$ 3.2} & \textbf{109.73 $\pm$ 3.9} \\
        \bottomrule
    \end{tabular}
    }
\end{small}

\end{table*}

\begin{table*}[t]
    \centering
    \caption{\textbf{Single-sample supervision ablation (median-target evaluation).} All predictors are trained with one sampled output length per prompt, but evaluated against the $16$-sample repeated-sampling median target. We report mean $\pm$ std test MAE over $8$ trials. Lower is better; the best per column is \textbf{bold}.}
    \label{tab:single_vs_median_main}

\begin{small}
    \setlength{\tabcolsep}{2.5pt}
    \renewcommand{\arraystretch}{1.04}
    \resizebox{\textwidth}{!}{%
    \begin{tabular}{lcccccccc}
        \toprule
        \multirow{2}{*}[-.5ex]{Method} & \multicolumn{4}{c}{Qwen} & \multicolumn{4}{c}{Llama} \\
        \cmidrule(lr){2-5}\cmidrule(l){6-9}
        & Math & Coding & LongSeq & Chat & Math & Coding & LongSeq & Chat \\
        \midrule
        $S^3$ & 43.91 $\pm$ 0.5 & 57.56 $\pm$ 0.2 & 75.65 $\pm$ 2.0 & 186.75 $\pm$ 3.6 & 25.76 $\pm$ 0.7 & 44.69 $\pm$ 2.8 & 55.57 $\pm$ 1.2 & 146.21 $\pm$ 3.8 \\
        \shortstack[l]{TRAIL-mean} & 49.84 $\pm$ 2.8 & 52.82 $\pm$ 2.6 & 80.64 $\pm$ 3.3 & 158.33 $\pm$ 5.3 & 27.16 $\pm$ 0.3 & 36.61 $\pm$ 1.8 & 60.07 $\pm$ 3.7 & 129.79 $\pm$ 4.4 \\
        \shortstack[l]{TRAIL-last} & 42.46 $\pm$ 9.6 & 49.47 $\pm$ 8.0 & 67.91 $\pm$ 1.2 & 138.46 $\pm$ 6.3 & 24.89 $\pm$ 0.2 & 35.23 $\pm$ 1.6 & 48.11 $\pm$ 0.9 & 102.48 $\pm$ 2.6 \\
        EGTP & 58.56 $\pm$ 9.2 & 56.90 $\pm$ 0.3 & 77.14 $\pm$ 3.0 & 266.55 $\pm$ 3.1 & 29.70 $\pm$ 1.1 & 40.80 $\pm$ 2.1 & 53.55 $\pm$ 1.0 & 156.07 $\pm$ 4.5 \\
        \textbf{ProD-M} & \textbf{33.56 $\pm$ 0.4} & \textbf{35.01 $\pm$ 1.2} & \textbf{65.44 $\pm$ 1.3} & \textbf{130.60 $\pm$ 2.6} & \textbf{20.79 $\pm$ 0.2} & \textbf{29.05 $\pm$ 1.0} & \textbf{43.45 $\pm$ 1.4} & \textbf{98.95 $\pm$ 1.3} \\
        \bottomrule
    \end{tabular}
    }
\end{small}

\end{table*}

\paragraph{Baselines.}
We compare against three external baselines: $S^3$~\citep{NeurIPS'23:s3}, EGTP~\citep{ICLR'26:EGTP}, and TRAIL~\citep{ICLR'25:embedding-scheduling}. We evaluate two offline TRAIL variants: TRAIL-Mean, which averages the final-layer hidden states across all input tokens in the prompt, and TRAIL-Last, which uses only the final-layer hidden state of the last input token as input to the length predictor. Our repeated-sampling approach yields two predictors: ProD-M, trained on the prompt-level median output length, and ProD-D, trained on a bin-projected empirical length distribution from the same repeated-sampling pool. We also report two reference lines: Noise Radius, the prompt-level median-centered noise radius, and Constant Median, which always predicts the train-split median length.

\paragraph{Metrics and protocol.}
We report mean absolute error (MAE) as the primary metric: for each test prompt, we compute the absolute difference between the predicted output length and the ground-truth label, and average over the entire test set. To construct stable supervision targets, we collect $16$ generations per prompt on both train and test splits using different random seeds; the ground-truth label for each prompt is defined as the median of these $16$ output lengths. The Noise Radius is accordingly defined as the average absolute deviation of the $16$ test-side generations from their per-prompt median. Although ProD-D is trained from a binned length-distribution label, it is evaluated against the same median-based target at test time, and all predictors are single-shot at inference. All methods use the official chat template from~\citet{ICLR'26:EGTP} and temperature $0.8$; further configuration details are in Appendix~\ref{sec:app_benchmark_protocol}.

\begin{figure*}[t]
    \centering
    \begin{subfigure}[t]{0.32\linewidth}
        \centering
        \includegraphics[width=\linewidth]{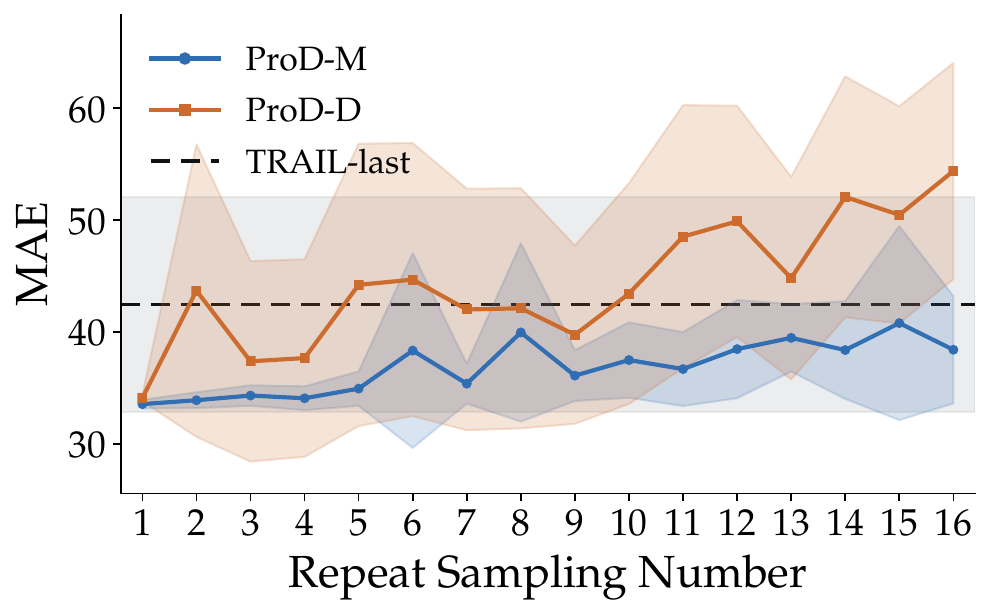}
        \caption{Qwen: Math}
    \end{subfigure}
    \hfill
    \begin{subfigure}[t]{0.32\linewidth}
        \centering
        \includegraphics[width=\linewidth]{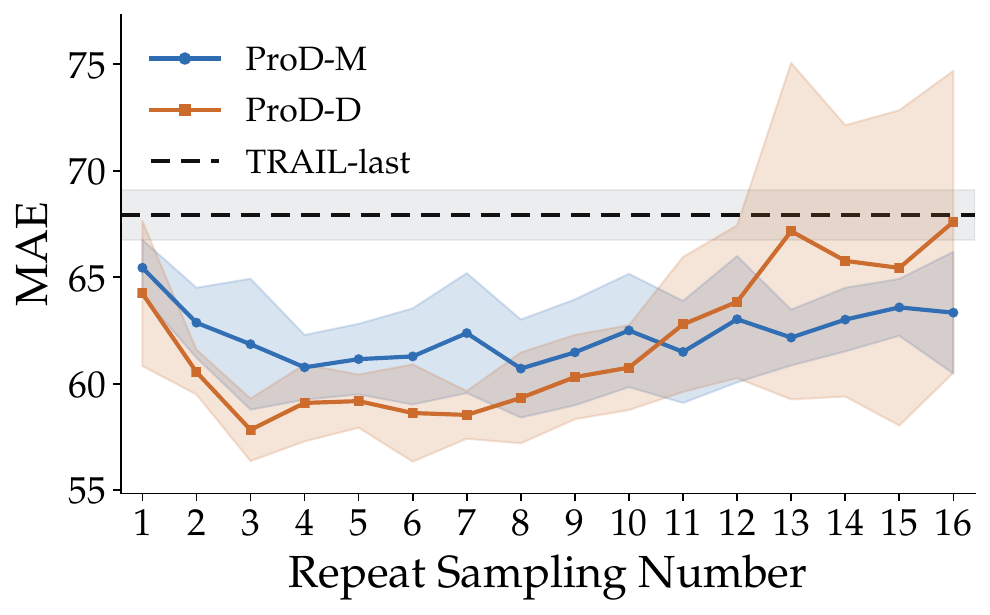}
        \caption{Qwen: LongSequence}
    \end{subfigure}
    \hfill
    \begin{subfigure}[t]{0.32\linewidth}
        \centering
        \includegraphics[width=\linewidth]{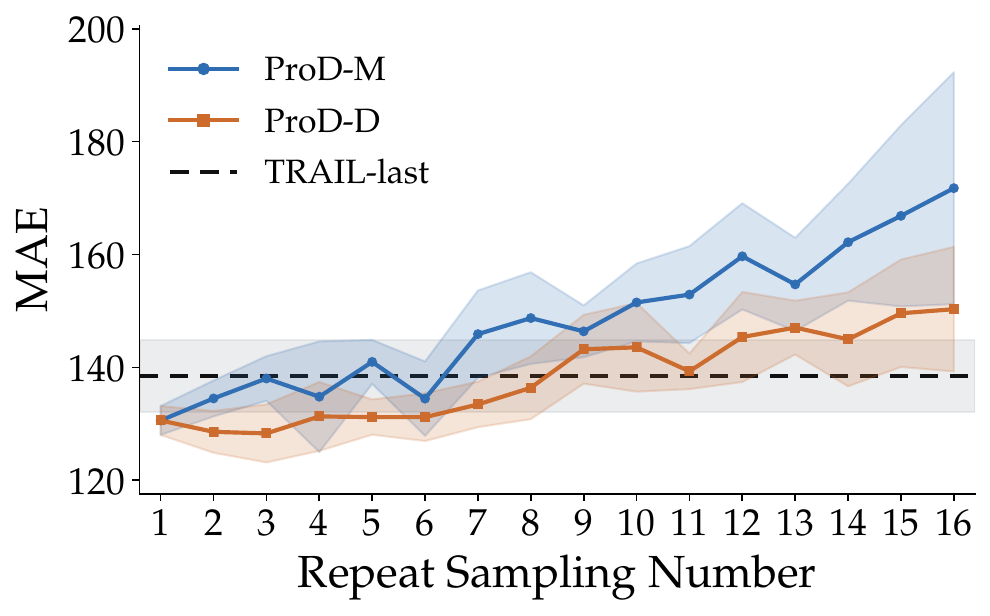}
        \caption{Qwen: Chat}
    \end{subfigure}

    \vspace{0.4em}

    \begin{subfigure}[t]{0.32\linewidth}
        \centering
        \includegraphics[width=\linewidth]{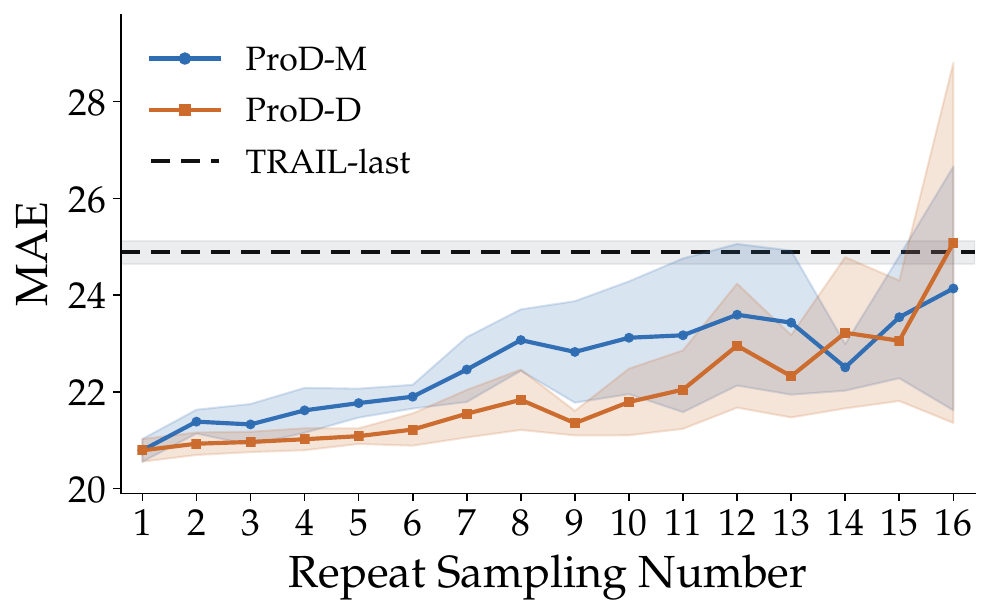}
        \caption{Llama: Math}
    \end{subfigure}
    \hfill
    \begin{subfigure}[t]{0.32\linewidth}
        \centering
        \includegraphics[width=\linewidth]{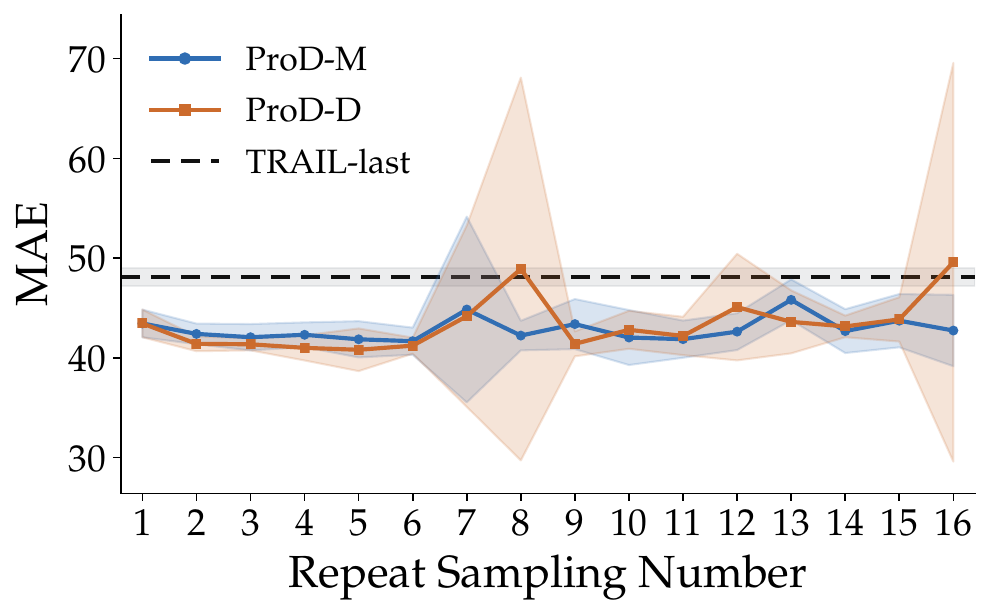}
        \caption{Llama: LongSequence}
    \end{subfigure}
    \hfill
    \begin{subfigure}[t]{0.32\linewidth}
        \centering
        \includegraphics[width=\linewidth]{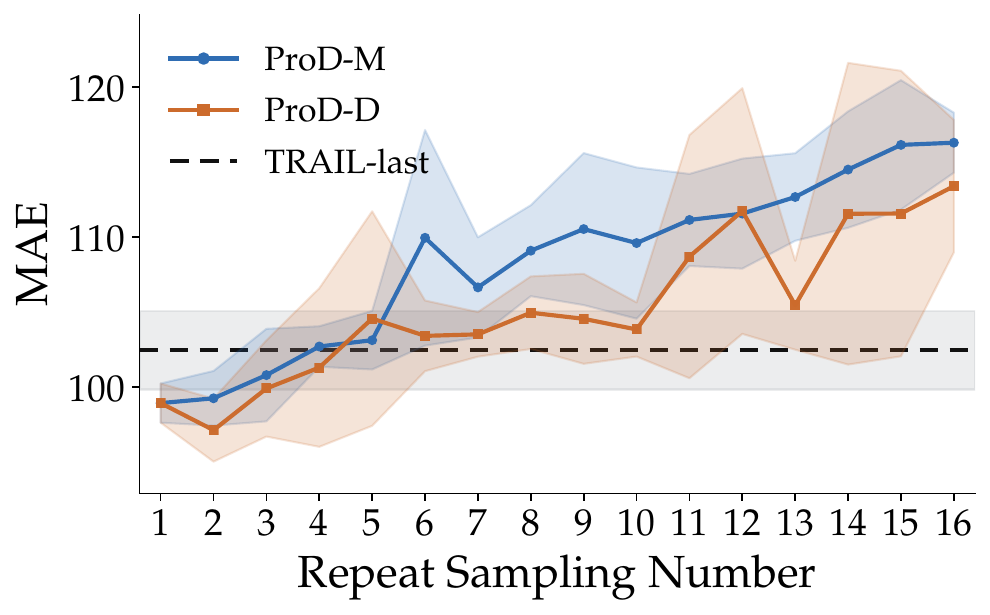}
        \caption{Llama: Chat}
    \end{subfigure}
    \caption{\textbf{Budget fairness: test MAE vs.\ repeat sampling number under a fixed inference budget.} As the repeat sampling number $k$ increases, only $\lceil B/k \rceil$ unique training prompts are retained. ProD-M and ProD-D are the repeated-sampling predictors; TRAIL-Last is the full-coverage single-sample baseline. All curves report mean $\pm$ std over $8$ trials. Coding is deferred to the appendix.}
    \label{fig:budget_fairness_repeat_curve_main}
\end{figure*}

\subsection{Output Length Prediction Evaluation}
We first evaluate prompt-only output length prediction. For a fair comparison, all methods, including both baselines and ours, are trained and evaluated under the same $16$-sample repeated-sampling protocol: for each prompt, we collect $16$ generations with different random seeds, and define the ground-truth label as the median of these $16$ output lengths. The reported metric is mean absolute error (MAE), i.e., the absolute difference between each method's prediction and the per-prompt median label, averaged over the test set. We additionally report median-centered noise radius reference, defined as the average absolute deviation of the $16$ test-side generations from their per-prompt median. This measures how much randomness the decoding process itself introduces, and we use it to analyze how close a predictor is to the best it can practically get: once a predictor's MAE reaches or falls below this level, the remaining error mostly comes from decoding randomness rather than the predictor itself. Table~\ref{tab:main_prompt_only_skeleton} summarizes the results; detailed hyperparameter grids, split manifests, and provenance notes are in Appendix~\ref{sec:app_benchmark_protocol}.

Across both models, ProD-D achieves the lowest average MAE: $56.78$ on Qwen and $44.31$ on Llama, outperforming the strongest external baseline TRAIL-Last ($64.23$ / $49.08$) by $11.6\%$ and $9.7\%$, respectively. On several scenarios ProD-D even falls below the Noise Radius, for example, $30.35$ vs.\ $32.94$ on Qwen-Math and $37.68$ vs.\ $56.38$ on Llama-LongSequence—indicating that repeated-sampling supervision enables the predictor to estimate the population median more accurately than the finite-sample average deviation. ProD-M, which uses the same repeated-sampling pool but learns from a simpler median label, is consistently second-best and still surpasses all external baselines. We also note that EGTP underperforms in most settings. In our implementation, we find that the entropy-weighted token selection tends to concentrate on a few early prompt tokens, whose hidden states may not encode the full prompt context; this likely explains its gap relative to methods that use the last-token embedding. Chat is the hardest scenario for all methods: even ProD-D has MAE $113.60$ on Qwen and $93.39$ on Llama, roughly $2{\times}$ the corresponding Noise Radius ($56.05$ / $45.71$). This gap reflects the high output-length uncertainty of open-ended conversational prompts.

\subsection{Single-sample Supervision Ablation}
To isolate the effect of repeated-sampling supervision, we retrain all predictor families using only a single sampled output length per prompt as the label, and repeat the full pipeline $8$ times to report mean $\pm$ std MAE. In this single-sample regime, ProD-D is omitted: a single length per prompt cannot induce a non-degenerate distribution target, so the distributional variant collapses to a relabeled point predictor. We therefore compare only ProD-M against the external baselines, providing the strictest test of how much noise is injected once repeated-sampling target construction is removed.

We report two complementary evaluations. Table~\ref{tab:single_vs_single_main} evaluates each single-sample-supervised predictor against the corresponding single-label test target, so both training and evaluation share the same one-shot noise. Table~\ref{tab:single_vs_median_main} keeps the same single-sample supervision on the training side but evaluates against the $16$-sample repeated-sampling median target, removing the high-variance test label and thereby isolating how much training-side target quality matters even when evaluation is stabilized.

Compared with Table~\ref{tab:main_prompt_only_skeleton}, single-sample supervision causes a clear degradation across both models. For example, ProD-M rises from $60.74$ to $80.93$ average MAE on Qwen and from $45.75$ to $63.28$ on Llama once repeated-sampling targets are replaced by one-shot labels (Table~\ref{tab:single_vs_single_main}). Table~\ref{tab:single_vs_median_main} confirms that the same trend persists even when evaluation switches back to the more stable median target. These results indicate that the gain of repeated-sampling supervision comes not merely from the predictor architecture, but from the robustness of the target construction itself.

\subsection{Budget Fairness and Repeat Curve}
Repeated-sampling supervision improves prediction quality, but it also consumes more inference during data collection. A natural fairness question arises: \emph{if the total number of training-side inferences is held fixed, does allocating part of that budget to repeated sampling still pay off compared to simply covering more unique prompts?}

To answer this, we fix the total training-observation budget and vary the repeat sampling number from $1$ to $16$. When the repeat sampling number is $k$, we retain only $\lceil B/k \rceil$ unique prompts (where $B$ is the canonical training-set size), each sampled $k$ times, so the total inference count stays budget-aligned. For example, on GSM8K ($B = 7473$), repeat sampling number $1$ means normal training on all $7473$ prompts with one sample each, while repeat sampling number $7$ reduces prompt coverage to only ${\sim}1068$ unique prompts with $7$ samples each. We compare ProD-M and ProD-D under this regime against TRAIL-Last trained on the full prompt set with single samples. All curves are evaluated against the $16$-sample median target and report mean $\pm$ std over $8$ trials.

Figure~\ref{fig:budget_fairness_repeat_curve_main} shows that repeated sampling remains valuable even under a fixed inference budget. As the repeat sampling number increases, the number of unique training prompts drops proportionally, yet the repeated-sampling predictors consistently outperform or match TRAIL-Last, which is trained on the full prompt set with single samples.

The effect is most pronounced on LongSequence: on Qwen, ProD-D with $3$ repeated samples (seeing only ${\sim}1263$ unique prompts) achieves $57.83$ MAE, well below the full-coverage TRAIL-Last baseline ($67.91$); on Llama, ProD-D with $5$ repeated samples reaches $40.80$, compared with $48.11$ for TRAIL-Last. A similar trend holds for Chat: ProD-D with $3$ repeated samples reaches $128.31$ on Qwen vs.\ $138.46$ for TRAIL-Last, and with $2$ repeated samples reaches $97.15$ on Llama vs.\ $102.48$ for TRAIL-Last.

On Math, the best results appear at repeat sampling number $1$, so this scenario does not exhibit an interior optimum. Nevertheless, even when the repeat sampling number reaches $7$ on GSM8K, reducing the training set from $7473$ to about $1068$ unique prompts, the repeated predictors remain competitive with the full-coverage TRAIL-Last baseline. Overall, these results confirm that repeated sampling acts as an effective budget-reallocation mechanism: concentrating repeated observations on fewer prompts can recover, or even surpass, the accuracy of single-sample methods trained on the full prompt set.

\section{Related Work and More Discussions}
\label{sec:related_work}

\paragraph{Output Length Prediction for LLM Serving.}
Output-length prediction is primarily motivated by the efficiency demands of LLM serving. In batched inference, shorter requests that finish early must wait for the longest request in the batch to complete, creating idle GPU cycles (``bubbles'')~\citep{NeurIPS'23:llmempowered}. Meanwhile, many serving frameworks reserve KV cache space based on the maximum possible output length to prevent out-of-memory failures; this wastes memory, constrains the batch size, and ultimately degrades GPU utilization and throughput~\citep{NeurIPS'23:s3}. Moreover, most serving systems default to first-come-first-served (FCFS) scheduling, under which a single long-running request can cause severe head-of-line blocking~\citep{NeurIPS'24:ltr}. These considerations have driven a growing line of research on output-length prediction.

\paragraph{Prompt-based Length Prediction.} Early work focuses on estimating output length directly from the input prompt before decoding begins. PO~\citep{NeurIPS'23:llmempowered} explores self-perception-based prediction by prompting or instruction-tuning LLMs to estimate their own response length prior to generation. $S^3$~\citep{NeurIPS'23:s3} trains a lightweight DistilBERT proxy to classify each request into predefined length buckets, using the predicted bucket for memory reservation and scheduling. Magnus~\citep{ICWS'24:Magnus} combines user-input length with application- and user-level semantic features via LaBSE embeddings and a random-forest regressor. SSJF~\citep{arXiv'24:SSJF} trains a BERT-based proxy predictor that formulates prompt-only output-length estimation as either regression or bucketized classification, and applies the prediction to shortest-job-first scheduling. LTR~\citep{NeurIPS'24:ltr} reformulates the problem as learning to rank, training a small OPT-based scorer with ListMLE to predict the relative order of request lengths rather than their absolute values, thereby realizing a more efficient scheduling policy SJF beyond FCFS.

\paragraph{Embedding-based Length Prediction.}More recent work leverages intermediate representations from the served LLM itself to build lighter-weight or online predictors. TRAIL~\citep{ICLR'25:embedding-scheduling} reuses the target LLM's intermediate embeddings for prompt-only length prediction and further refines the remaining-length estimate online during decoding. ELIS~\citep{arXiv'25:elis} couples an encoder-based response-length predictor with an iterative shortest-remaining-time-first scheduler that repeatedly re-prioritizes requests by estimated remaining output length. \citet{ACL'25:token-end-graph} model layerwise hidden states as a graph and apply a GCN-based regressor to predict the number of remaining tokens. EGTP~\citep{ICLR'26:EGTP} aggregates prompt representations via entropy-guided pooling over the served LLM's hidden states for output-length prediction, while its progressive variant PLP extends this idea to remaining-length prediction during decoding.

\noindent \textbf{More discussion.} We note the concurrent work of \citet{arXiv'26:uncertainty-scheduling}, which we became aware of after completing the first version of this paper.
They also notice that LLM output length is better modeled as a stochastic distribution with heavy-tailed characteristics. Indeed, we appreciate their even more detailed empirical examination of heavy-tailed length distributions. 
Nonetheless, the two works differ fundamentally in methodology: their approach fits a parametric distribution to repeated output lengths and learns the distribution parameters, whereas ours uses repeated sampling to construct supervision targets directly from the empirical prompt-level length distribution, with a particular focus on robust target design for prompt-only prediction. 

\section{Conclusion}
\label{sec:conclusion}

In this paper, we observe a fundamental yet overlooked mismatch in output-length prediction for LLM serving: for the LLM decoding process, prompt-conditioned output length is drawn from a distribution rather than being a deterministic scalar. Existing methods typically treat one sampled completion length as the ground-truth label, targeting either stronger representations or finer prediction heads while leaving the supervision target itself unexamined. Moreover, the prompt-conditioned length distribution often shows patterns consistent with heavy-tailed behavior, where occasional long generations shift the single-sample label far from the stable central tendency, yielding a statistically misaligned training signal. To this end, we propose the \emph{prompt-conditioned length distribution} (ProD) method, which constructs robust training targets from repeated independent generations of the same prompt and instantiates two variants that balance supervision robustness with inference efficiency: ProD-M compresses repeated observations into a median-based point target aligned with the Bayes-optimal prediction under MAE, while ProD-D projects sampled lengths onto a binned histogram that preserves the full prompt-conditioned uncertainty. At inference time, both variants reuse the served LLM's last-layer hidden states and remain single-shot. Under a simplified linear surrogate, we further provide a theoretical guarantee showing that increasing the repeated-sampling budget exponentially reduces the failure probability of the estimation bound. Experiments on Qwen-2.5-7B and Llama-3-8B across math, coding, long-sequence, and chat scenarios demonstrate that ProD-D achieves the lowest average MAE on both backbones, reducing prediction error by up to $25\%$ over previous state-of-the-art methods, and that this advantage persists even under fixed total inference budgets, confirming that the gain stems from the supervision target itself rather than from additional predictor complexity.

This work focuses on the prompt-only setting. Our formulation in Section~\ref{sec:problem_formulation} already accommodates the general case where the predictor refines its remaining-length estimate iteratively as decoding progresses. For the next step, we will extend our approach to this iterative regime and integrate the resulting predictions into serving scenarios~\citep{SOSP'23:pagedattention} for learning-augmented scheduling.

\bibliographystyle{icml2026}
\bibliography{core}


\newpage
\appendix
\onecolumn
\section{Experimental Details}
\label{sec:app_exp}

\subsection{Detailed Metrics}
\paragraph{Metrics for median-centered noise radius.}
Repeated sampling from the same prompt induces a distribution over output lengths rather than a single deterministic value. We refer to the average absolute deviation of these repeated lengths from their median as the \emph{noise radius}, which measures the randomness introduced by the decoding process itself. When a predictor's MAE approaches or falls below this level, the remaining error is mainly due to decoding randomness rather than the predictor's own limitation. For prompt $i$, suppose we collect $R$ repeated generations with output lengths $\{L_{i,r}\}_{r=1}^R$. We define its noise radius using the median-based mean absolute error (Median-MAE). Let $m_i \define \mathrm{median}\mbr{\bbr{L_{i,r}}_{r=1}^{R}}$
denote the sample median length for prompt $i$. The corresponding median-based MAE is defined as
\begin{equation*}
  \mathrm{MAE}^{(\mathrm{median})}_i
  ~=~ \frac{1}{R}\sum_{r=1}^{R}\abs{L_{i,r}-m_i}.
\end{equation*}

We use the median, rather than the mean or the average pairwise MAE, because output-length distributions are often heavy-tailed. The median provides a more robust measure of central tendency and is less sensitive to a small number of unusually long generations. Therefore, Median-MAE more faithfully captures the intrinsic prompt-level variability and serves as a more stable estimate of the noise radius.

\begin{figure}[!h]
    \centering
    \begin{subfigure}[t]{0.32\linewidth}
        \centering
        \includegraphics[width=\linewidth]{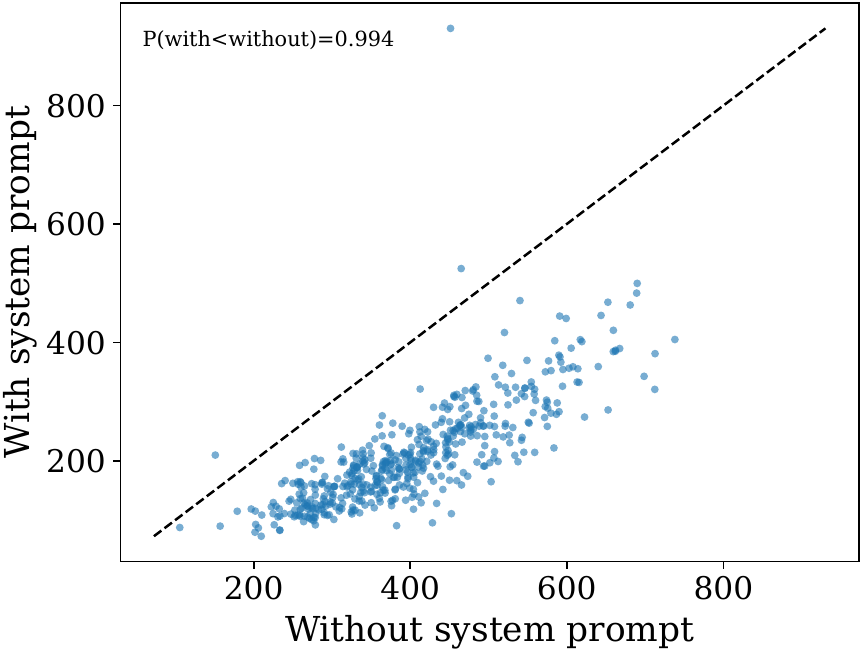}
        \caption{Mean length}
        \label{fig:sys_prompt_mean_paired}
    \end{subfigure}
    \hfill
    \begin{subfigure}[t]{0.32\linewidth}
        \centering
        \includegraphics[width=\linewidth]{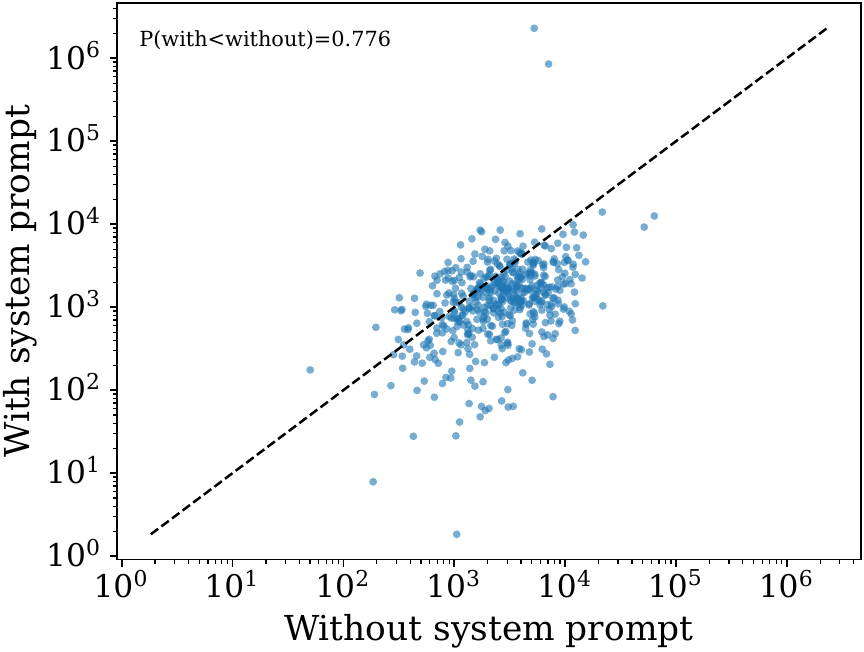}
        \caption{Length variance (log).}
        \label{fig:sys_prompt_var_paired_loglog}
    \end{subfigure}
    \hfill
    \begin{subfigure}[t]{0.32\linewidth}
        \centering
        \includegraphics[width=\linewidth]{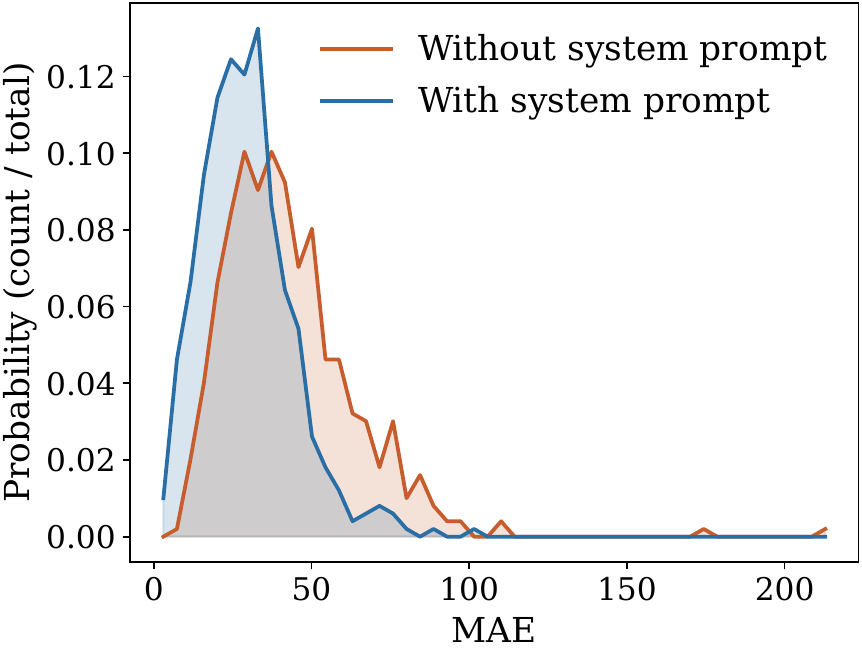}
        \caption{Mean-MAE distribution.}
        \label{fig:sys_prompt_mean_mae_dist}
    \end{subfigure}

    \vspace{0.6em}

    \begin{subfigure}[t]{0.32\linewidth}
        \centering
        \includegraphics[width=\linewidth]{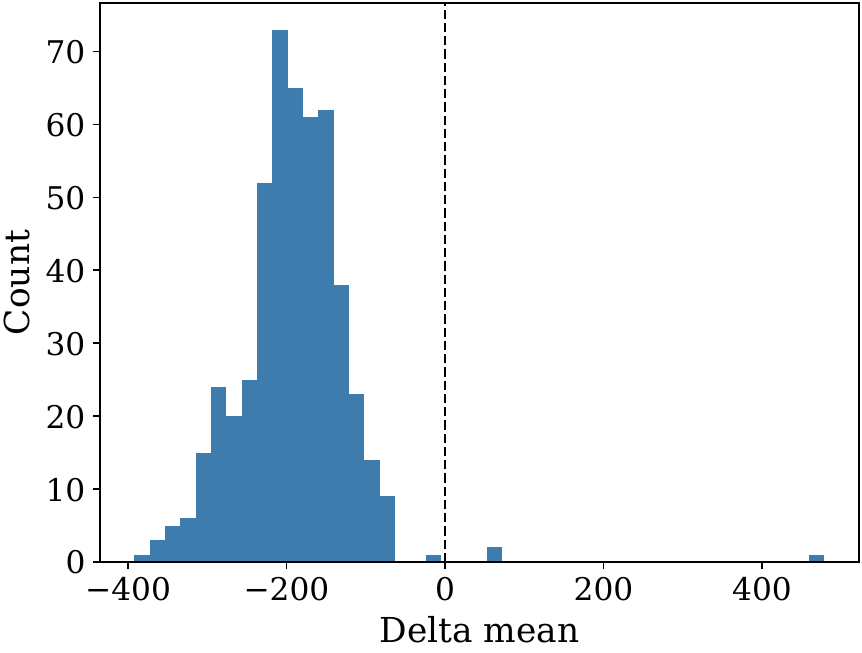}
        \caption{$\Delta$ mean length.}
        \label{fig:sys_prompt_delta_mean_hist}
    \end{subfigure}
    \hfill
    \begin{subfigure}[t]{0.32\linewidth}
        \centering
        \includegraphics[width=\linewidth]{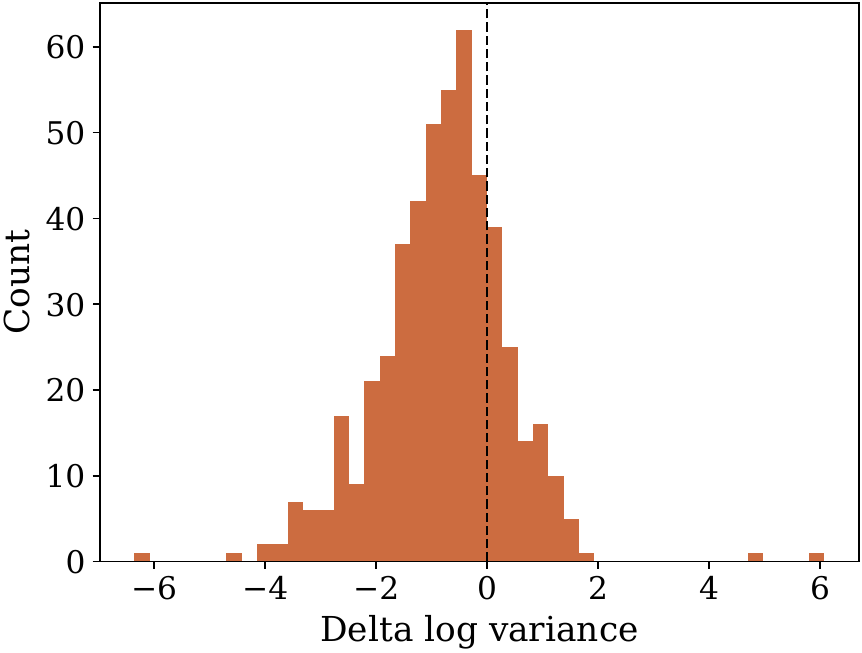}
        \caption{$\Delta \log$ variance.}
        \label{fig:sys_prompt_delta_logvar_hist}
    \end{subfigure}
    \hfill
    \begin{subfigure}[t]{0.32\linewidth}
        \centering
        \includegraphics[width=\linewidth]{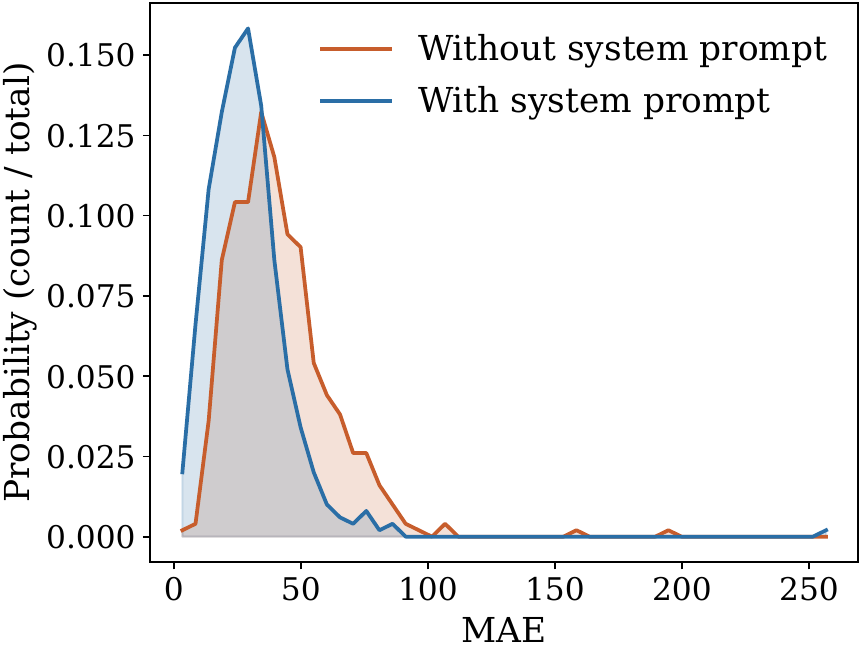}
        \caption{Median-MAE distribution.}
        \label{fig:sys_prompt_median_mae_dist}
    \end{subfigure}

    \vspace{0.6em}

    \begin{subfigure}[t]{\linewidth}
        \centering
        \includegraphics[width=0.98\linewidth]{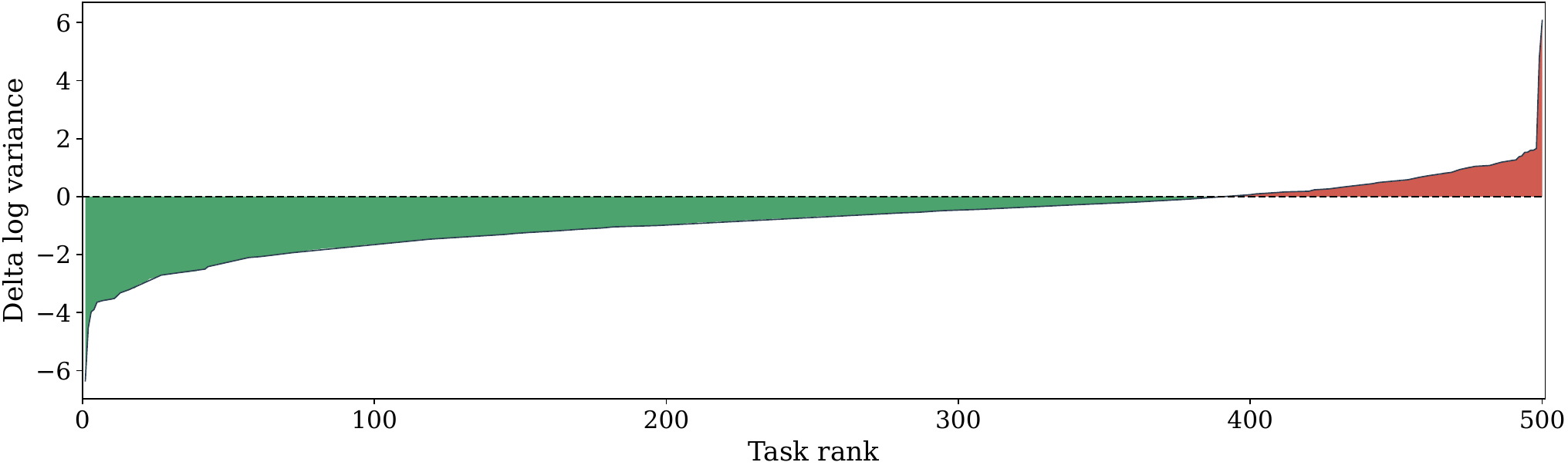}
        \caption{Waterfall of $\Delta \log$ variance per prompt (sorted).}
        \label{fig:sys_prompt_variance_waterfall}
    \end{subfigure}

    \caption{\textbf{System prompt reduces output-length randomness and MAE noise radius.}
    Qwen2.5-7B-Instruct on 500 MBPP prompts with 16 trials per prompt (8 with system prompt, 8 without).
    We compare per-prompt mean length, length variance, and two MAE-style dispersion measures (Mean-MAE / Median-MAE).
    Paired plots (Figure~\ref{fig:sys_prompt_mean_paired} and Figure~\ref{fig:sys_prompt_var_paired_loglog}) and shift summaries (Figure~\ref{fig:sys_prompt_delta_mean_hist}, Figure~\ref{fig:sys_prompt_delta_logvar_hist}, and Figure~\ref{fig:sys_prompt_variance_waterfall}) show that adding the system prompt typically shortens outputs and often reduces variance.
    MAE distributions (Figure~\ref{fig:sys_prompt_mean_mae_dist} and Figure~\ref{fig:sys_prompt_median_mae_dist}) shift left and become more concentrated, indicating a lower and tighter per-prompt MAE baseline.}
    \label{fig:1-to-many-mbpp}
\end{figure}

\subsection{Detailed Benchmark and Split Protocol}
\label{sec:app_benchmark_protocol}
We organize our benchmark into four semantic scenarios: chat, long-sequence, reasoning, and coding. Specifically, we use LMSYS-Chat-1M~\citep{ICLR'24:lmsys-chat-1m} represent chat scenario, LongBench~\citep{ACL'24:longbench} represent long-sequence scenario, GSM8K~\citep{arXiv'21:training-verifiers-math} represent reasoning scenario, and MBPP~\citep{arXiv'21:program-synthesis-llm} represent coding scenario.

\paragraph{Dataset split details.}
To ensure a fair evaluation, the main prompt-only benchmark freezes a single anchor dataset per scenario: GSM8K for Math, MBPP for Coding, LongBench for LongSequence, and LMSYS-Chat-1M for Chat. We use the official train/test splits for GSM8K ($7473/1319$) and MBPP ($374/500$). For LongBench, we derive deterministic train/test splits, which yields $3789/961$ on Qwen and $3780/970$ on Llama. For LMSYS-Chat-1M, we derive deterministic splits with $4070/930$ prompts on both models. 

\paragraph{Inference details.}
We use a fixed prompting protocol throughout all experiments. For each served model, we apply its official chat template together with the system prompt provided by EGTP~\citep{ICLR'26:EGTP}. Unless otherwise stated, all experiments use zero-shot prompting, temperature $0.8$, and $R=16$ repeated generations per prompt with different random seeds on both train and test. The point target for each prompt is the median of the $16$ sampled output lengths. ProD-D instead projects those $16$ lengths onto a predefined bin grid and learns from the resulting discrete empirical distribution, but it is still evaluated with test MAE against the same test-side median label. Noise radius reports the average absolute deviation of the $16$ test generations from their prompt-level median, and Constant Median predicts the train-split median length for all test prompts. GSM8K, MBPP, and LMSYS-Chat-1M use the standard repeated-sampling path, while LongBench uses a token-aware truncation path with reserved generation budget before the same downstream aggregation, split derivation, and predictor stages.

\begin{figure*}[t]
    \centering
    \begin{subfigure}[t]{0.49\linewidth}
        \centering
        \includegraphics[width=\linewidth]{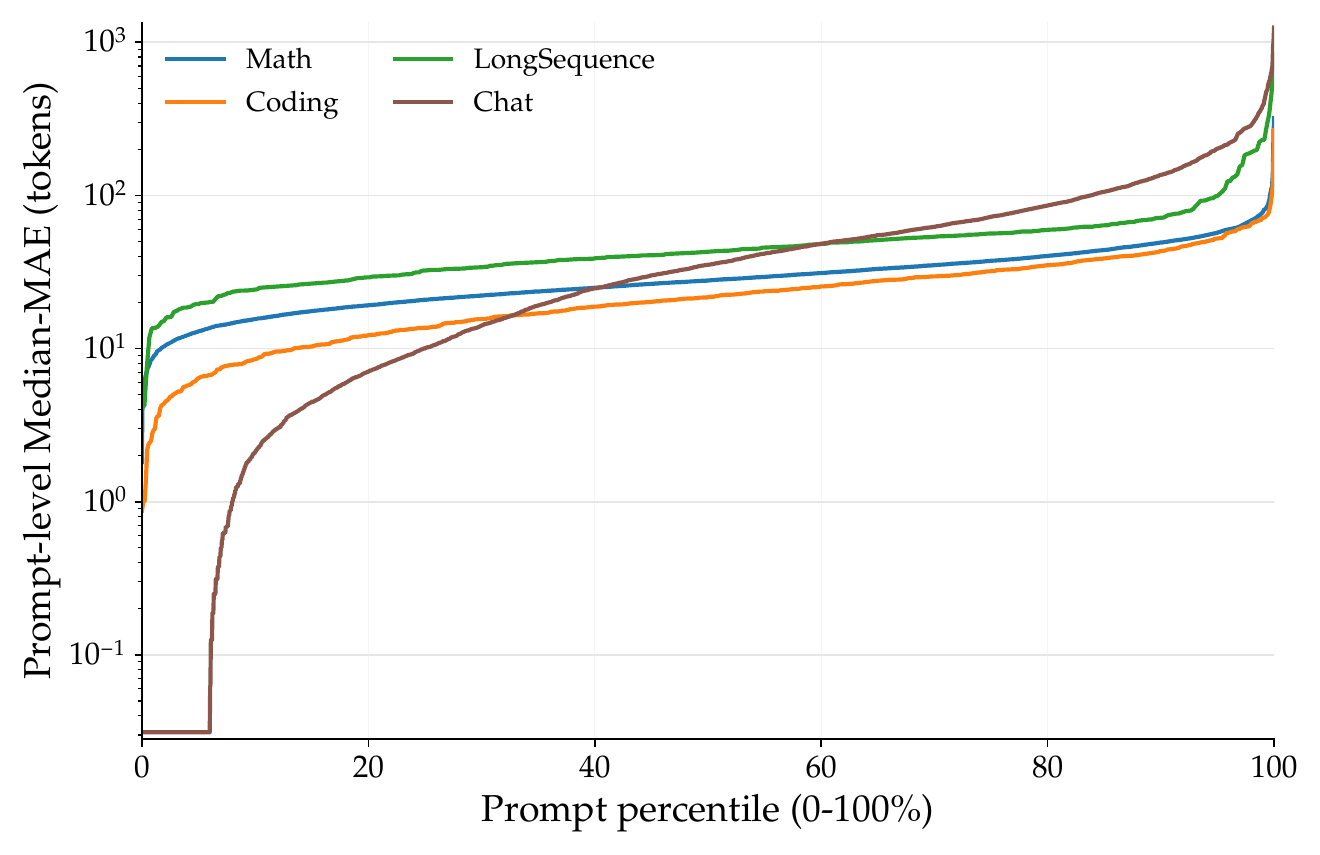}
        \caption{Qwen prompt-percentile waterfall.}
    \end{subfigure}
    \hfill
    \begin{subfigure}[t]{0.49\linewidth}
        \centering
        \includegraphics[width=\linewidth]{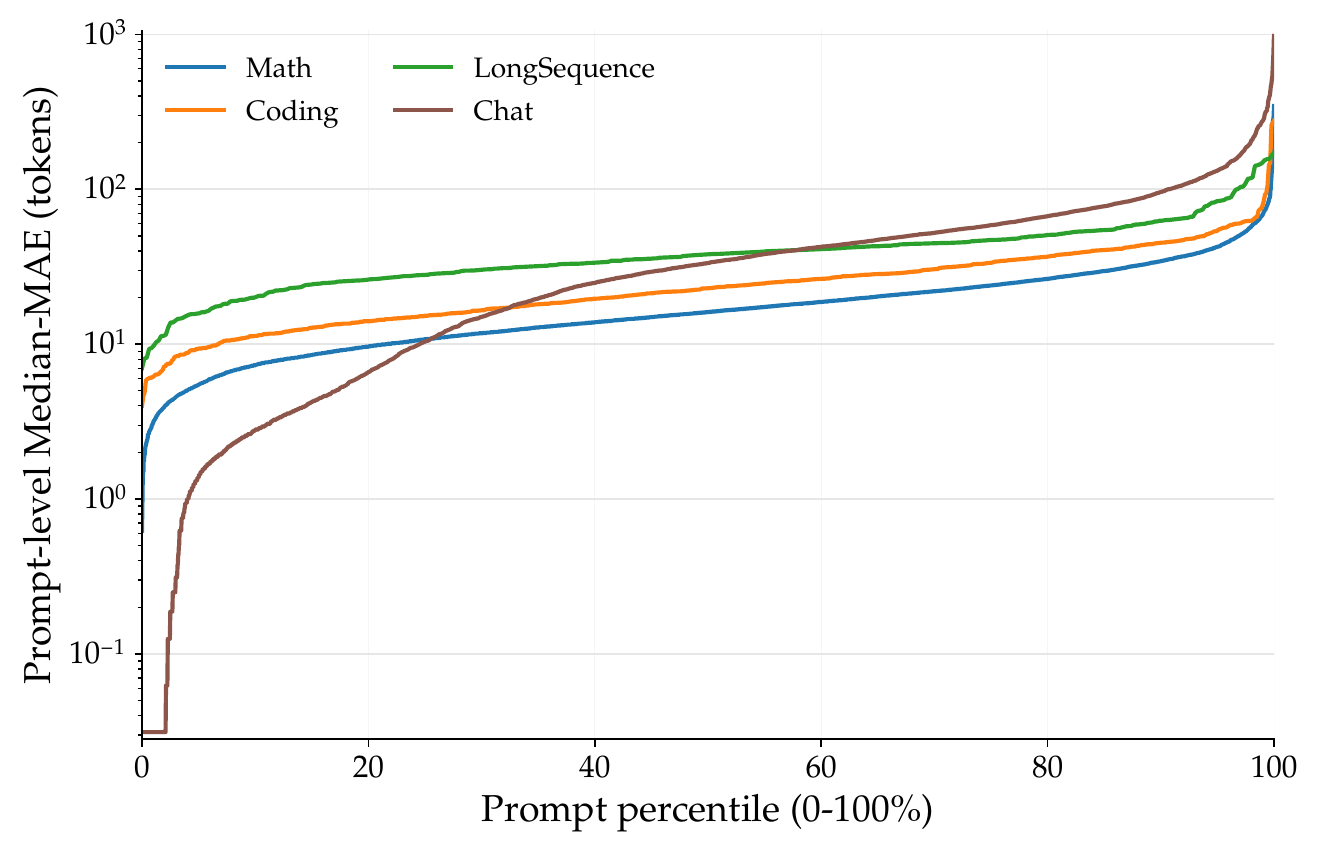}
        \caption{Llama prompt-percentile waterfall.}
    \end{subfigure}
    \caption{\textbf{Prompt-percentile noise-floor waterfalls.} Each curve sorts prompts by prompt-level Median-MAE within the corresponding model and plots the values on a log-scale y-axis against prompt percentile. The plotting floor is used only to visualize zero-valued prompts on the log scale: $300$ prompts for Qwen and $105$ prompts for Llama.}
    \label{fig:app_key_observations_waterfalls}
\end{figure*}

\paragraph{Predictor search spaces and provenance.}
\begingroup
\sloppy
For $S^3$, we sweep num\_bins over \{7, 10, 13, 15, 20\} and use a scene- and model-specific 10-point bin\_max grid. The Qwen bin\_max candidates are \{621, 690, 759, 828, 898, 967, 1036, 1105, 1174, 1243\} for GSM8K, \{400, 444, 488, 533, 577, 621, 666, 710, 755, 799\} for MBPP, \{1631, 1812, 1994, 2175, 2356, 2537, 2719, 2900, 3081, 3262\} for LongBench, and \{3296, 3663, 4029, 4395, 4761, 5128, 5494, 5860, 6226, 6593\} for LMSYS-Chat-1M. The corresponding Llama grids are \{469, 521, 573, 626, 678, 730, 782, 834, 886, 938\}, \{433, 481, 529, 577, 626, 674, 722, 770, 818, 866\}, \{1344, 1494, 1643, 1792, 1942, 2091, 2240, 2390, 2539, 2689\}, and \{2211, 2456, 2702, 2948, 3193, 3439, 3685, 3930, 4176, 4422\} in the same scenario order. The selected S3 configurations are (20, 759), (10, 400), (13, 1631), and (13, 3296) on Qwen, and (15, 678), (7, 481), (13, 1344), and (15, 2211) on Llama.

For EGTP, we sweep num\_bins over \{7, 10, 13, 15, 20\} and $\lambda$ over \{0.50, 0.55, 0.60, 0.65, 0.70, 0.75, 0.80, 0.85, 0.90, 0.95\}. The selected (num\_bins, $\lambda$) configurations are (7, 0.70), (20, 0.95), (10, 0.55), and (10, 0.55) on Qwen, and (7, 0.95), (7, 0.65), (15, 0.70), and (10, 0.95) on Llama, again in the order Math/Coding/LongSequence/Chat.
\par
\fussy
\endgroup

\subsection{Dataset Length Distribution and Noise Analysis}

\paragraph{The impact of system prompt.}
In this part, we verify that adding a fixed system prompt makes output length more predictable under repeated sampling. We use the system prompt from~\citet{ICLR'26:EGTP}, the MBPP dataset, and Qwen2.5-7B-Instruct. We randomly sample 500 prompts and run 16 independent generations per prompt, split into 8 trials \emph{with} the system prompt and 8 trials \emph{without}, keeping the decoding configuration the same. For each prompt, we compute (i) mean output length, (ii) length variance, and (iii) two MAE-based dispersion measures over the 16 sampled lengths: Mean-MAE (absolute deviation from the sample mean) and Median-MAE (absolute deviation from the sample median). Figure~\ref{fig:sys_prompt_mean_paired} shows that the system prompt typically shortens the mean output length, and Figure~\ref{fig:sys_prompt_var_paired_loglog} shows that it often reduces the variance of sampled lengths.
The corresponding shift histograms in Figure~\ref{fig:sys_prompt_delta_mean_hist} and Figure~\ref{fig:sys_prompt_delta_logvar_hist}, together with the waterfall in Figure~\ref{fig:sys_prompt_variance_waterfall}, indicate that these reductions hold for a large fraction of prompts.
More importantly for MAE-based length prediction, the Mean-MAE and Median-MAE distributions in Figure~\ref{fig:sys_prompt_mean_mae_dist} and Figure~\ref{fig:sys_prompt_median_mae_dist} shift left and become more concentrated when adding the system prompt.
Overall, adding a system prompt reduces sampling-induced randomness in output length and lowers the per-prompt MAE baseline, which increases the headroom for achieving lower MAE in downstream length predictors.

\begin{figure*}[t]
    \centering
    \begin{subfigure}[t]{0.24\linewidth}
        \centering
        \includegraphics[width=\linewidth]{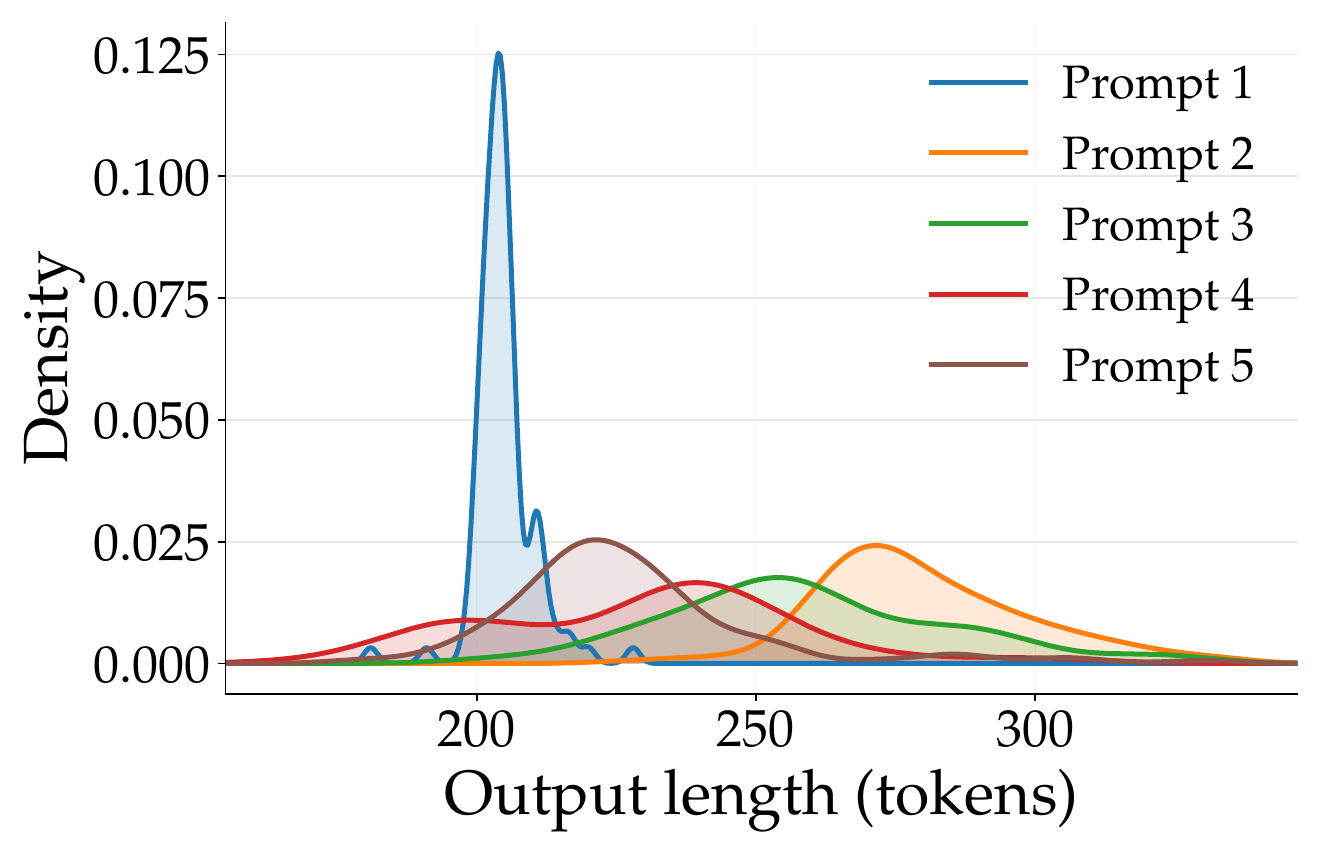}
        \caption{Math light5}
    \end{subfigure}
    \hfill
    \begin{subfigure}[t]{0.24\linewidth}
        \centering
        \includegraphics[width=\linewidth]{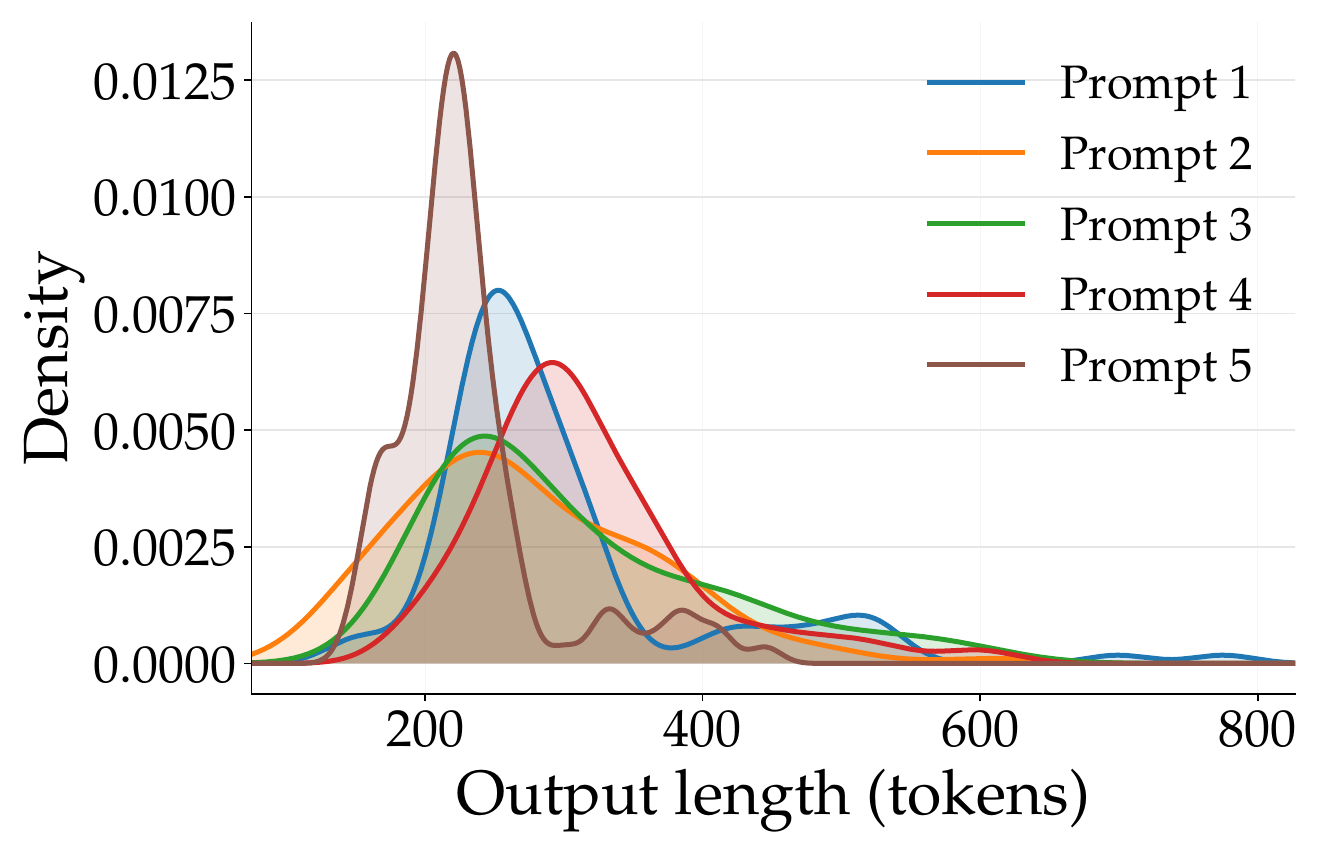}
        \caption{Math heavy5}
    \end{subfigure}
    \hfill
    \begin{subfigure}[t]{0.24\linewidth}
        \centering
        \includegraphics[width=\linewidth]{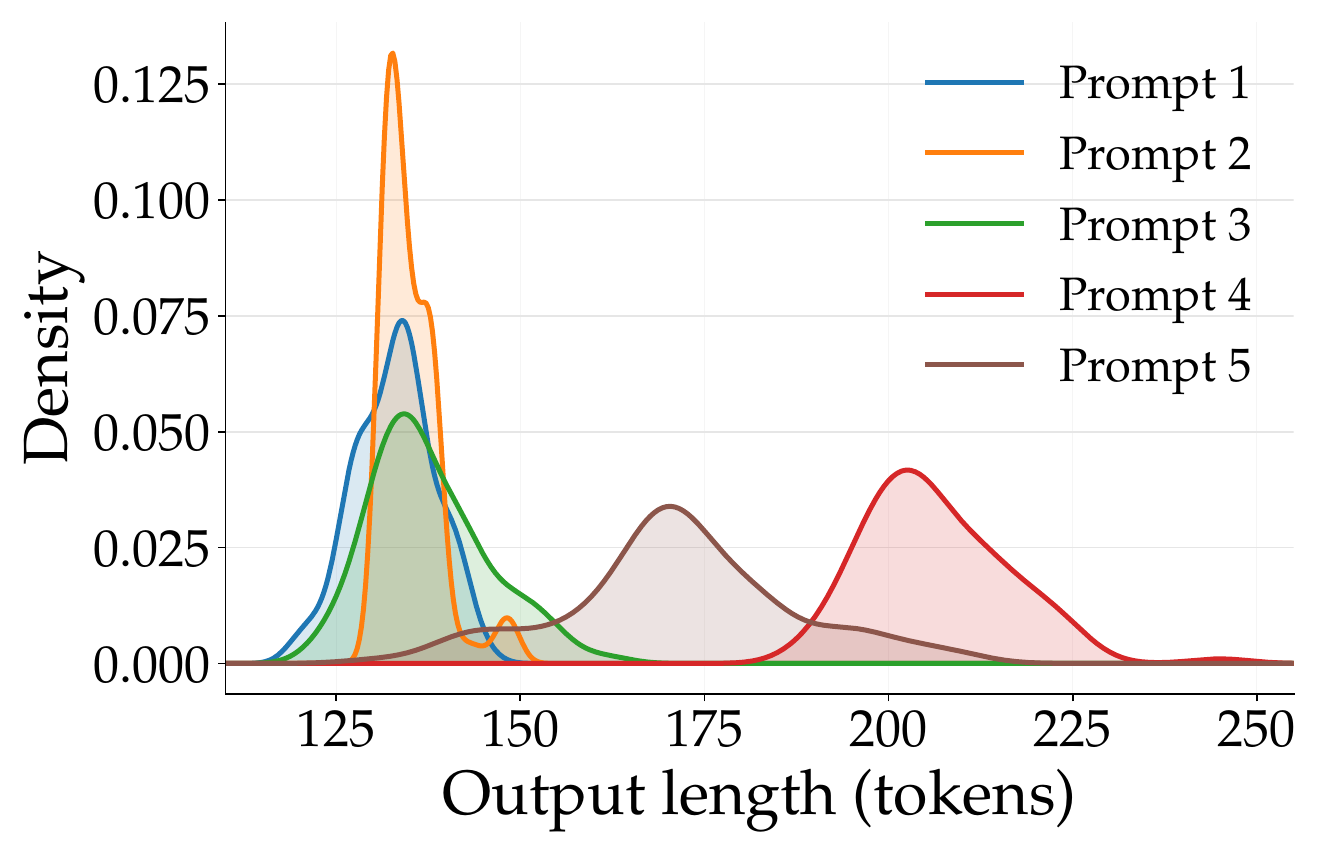}
        \caption{Coding light5}
    \end{subfigure}
    \hfill
    \begin{subfigure}[t]{0.24\linewidth}
        \centering
        \includegraphics[width=\linewidth]{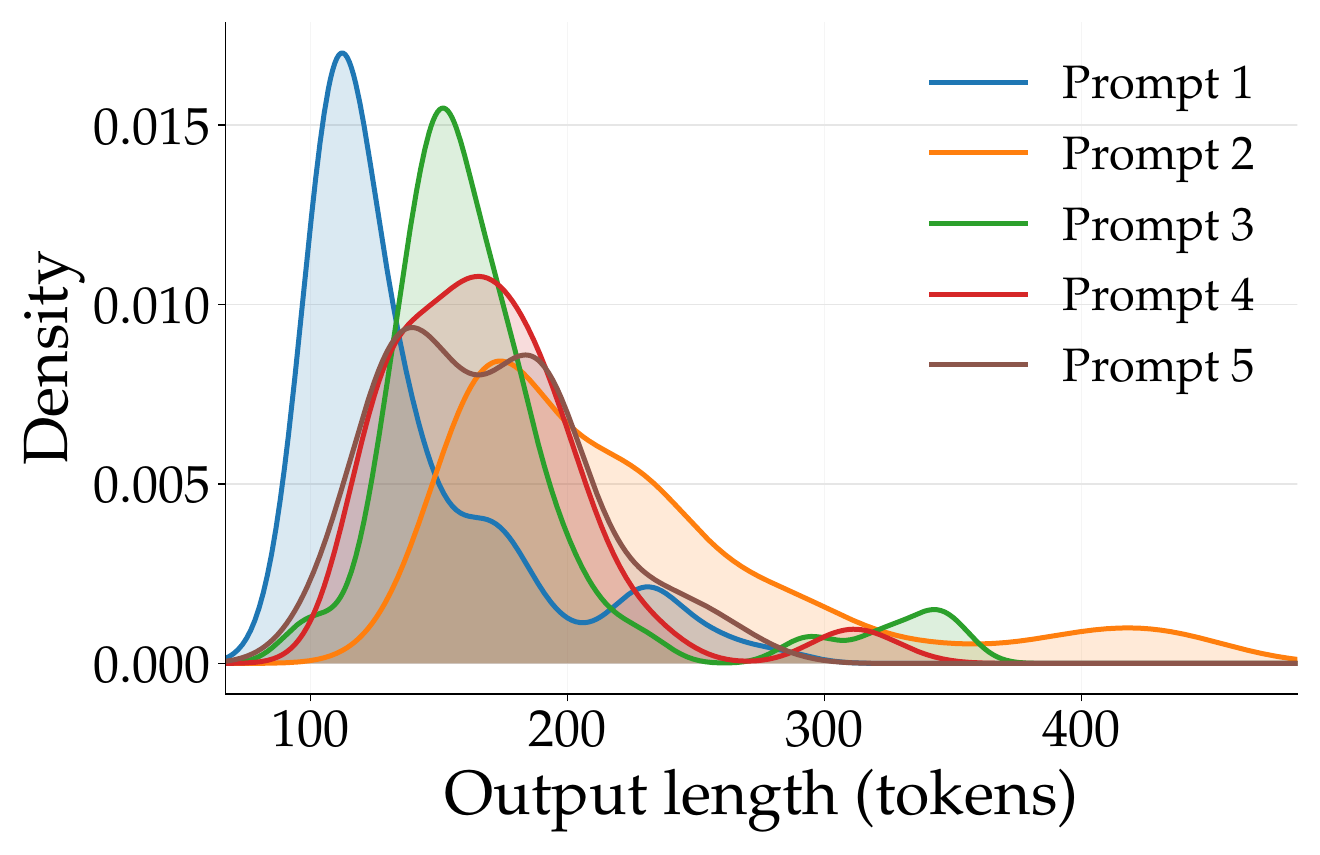}
        \caption{Coding heavy5}
    \end{subfigure}

    \vspace{0.5em}

    \begin{subfigure}[t]{0.24\linewidth}
        \centering
        \includegraphics[width=\linewidth]{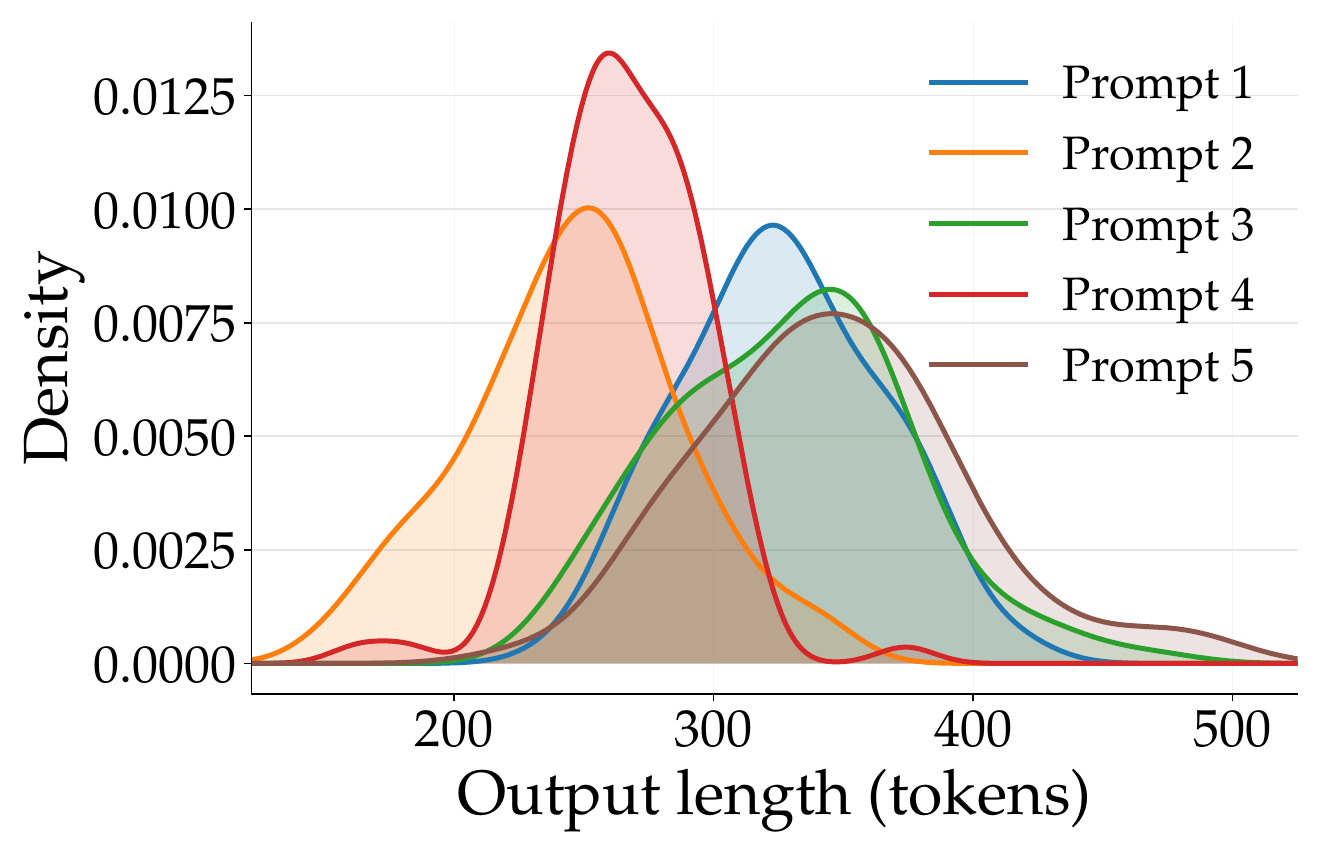}
        \caption{LongSequence light5}
    \end{subfigure}
    \hfill
    \begin{subfigure}[t]{0.24\linewidth}
        \centering
        \includegraphics[width=\linewidth]{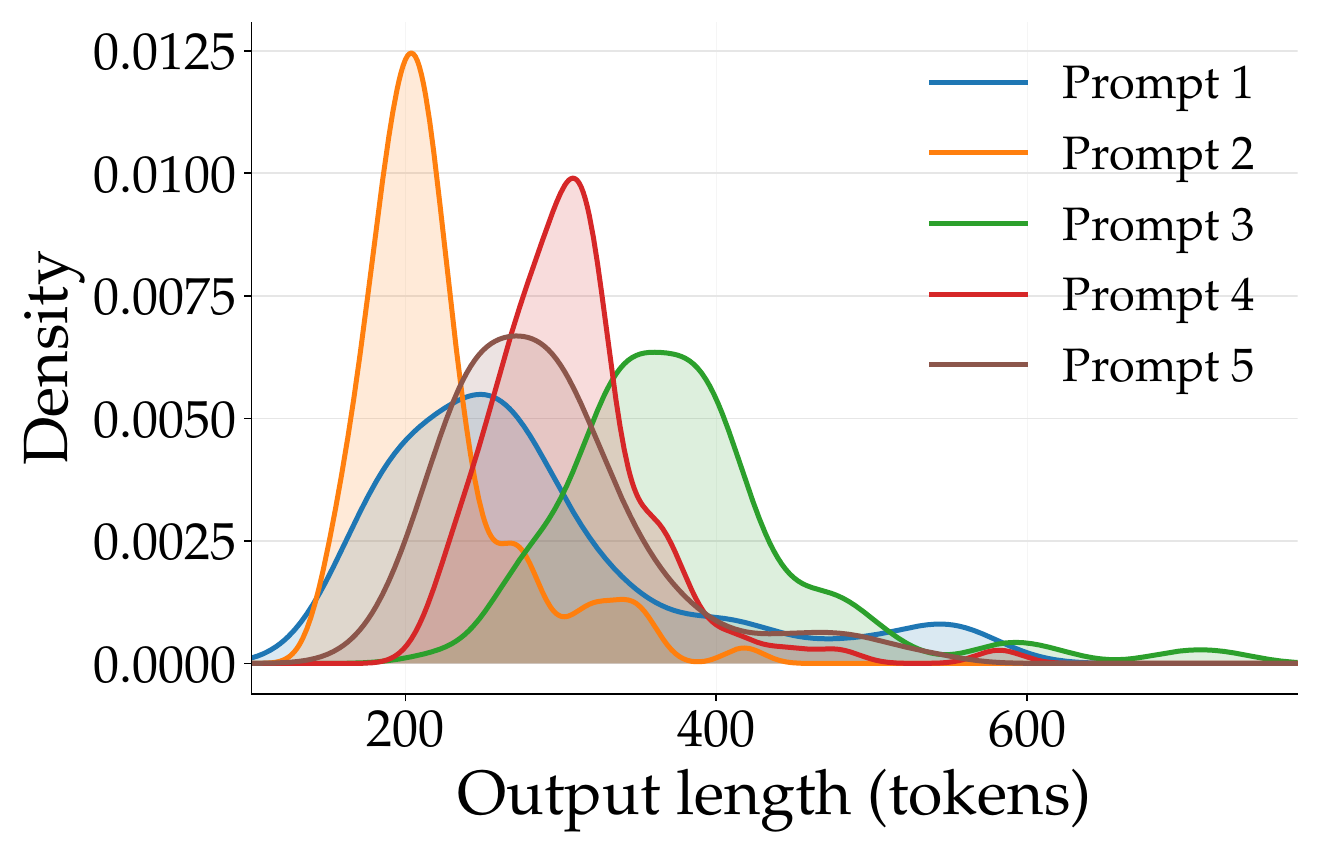}
        \caption{LongSequence heavy5}
    \end{subfigure}
    \hfill
    \begin{subfigure}[t]{0.24\linewidth}
        \centering
        \includegraphics[width=\linewidth]{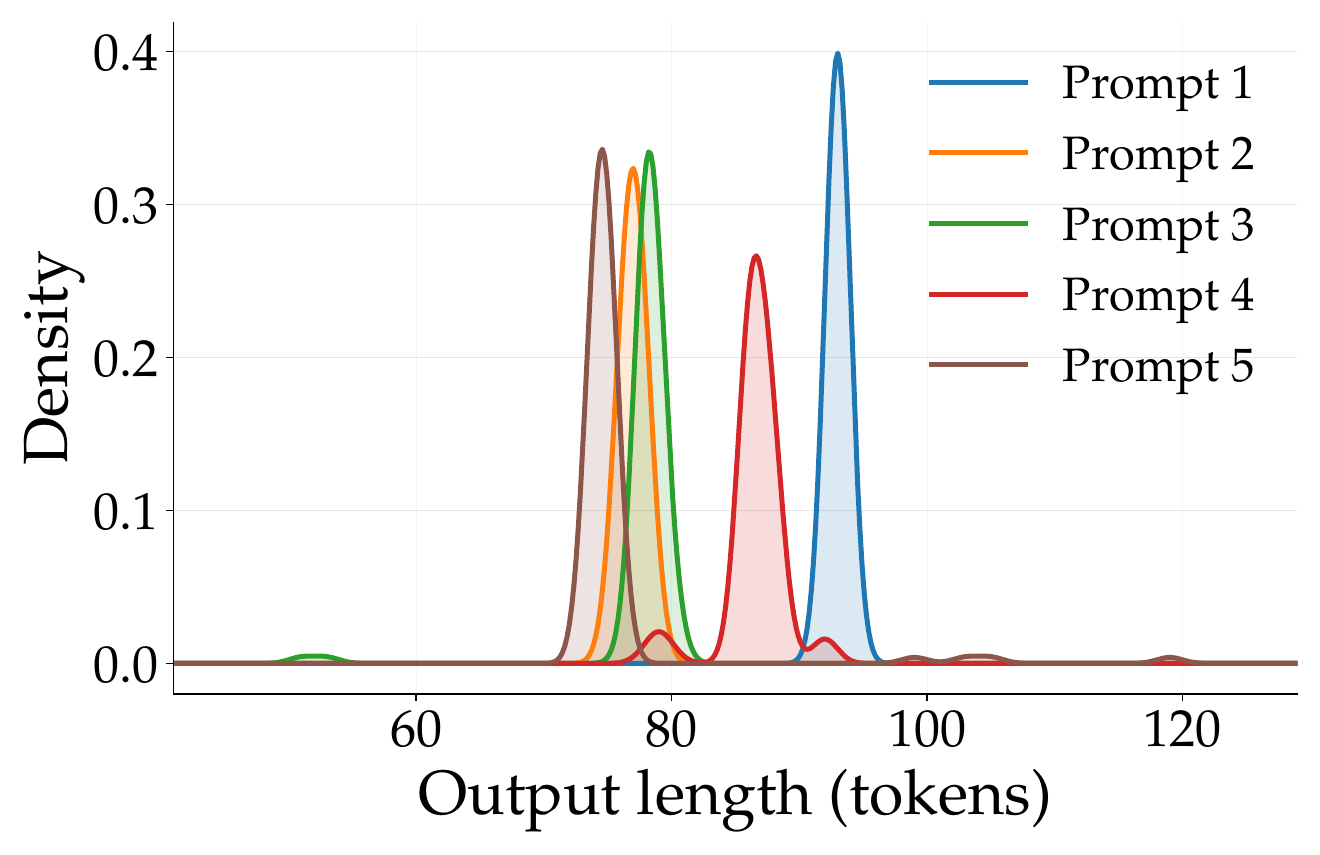}
        \caption{Chat light5}
    \end{subfigure}
    \hfill
    \begin{subfigure}[t]{0.24\linewidth}
        \centering
        \includegraphics[width=\linewidth]{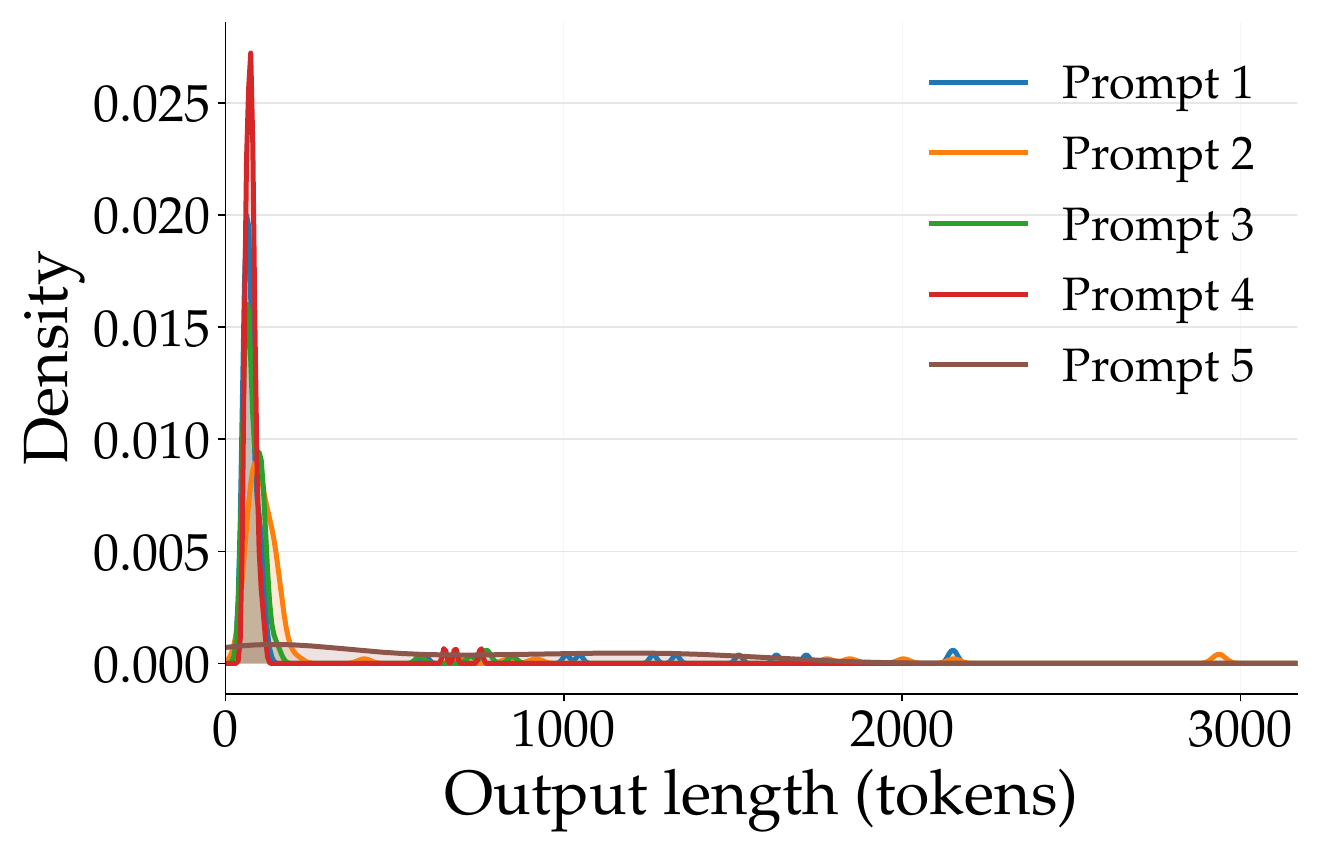}
        \caption{Chat heavy5}
    \end{subfigure}
    \caption{\textbf{Qwen per-setting light/heavy overlays.} For each setting, \emph{light5} and \emph{heavy5} denote the five prompts with the smallest and largest $\max(\mathrm{length}) / \mathrm{median}(\mathrm{length})$ among the ten repeated-sampling prompts.}
    \label{fig:app_key_observations_qwen_overlays}
\end{figure*}

\begin{figure*}[!htb]
    \centering
    \begin{subfigure}[t]{0.24\linewidth}
        \centering
        \includegraphics[width=\linewidth]{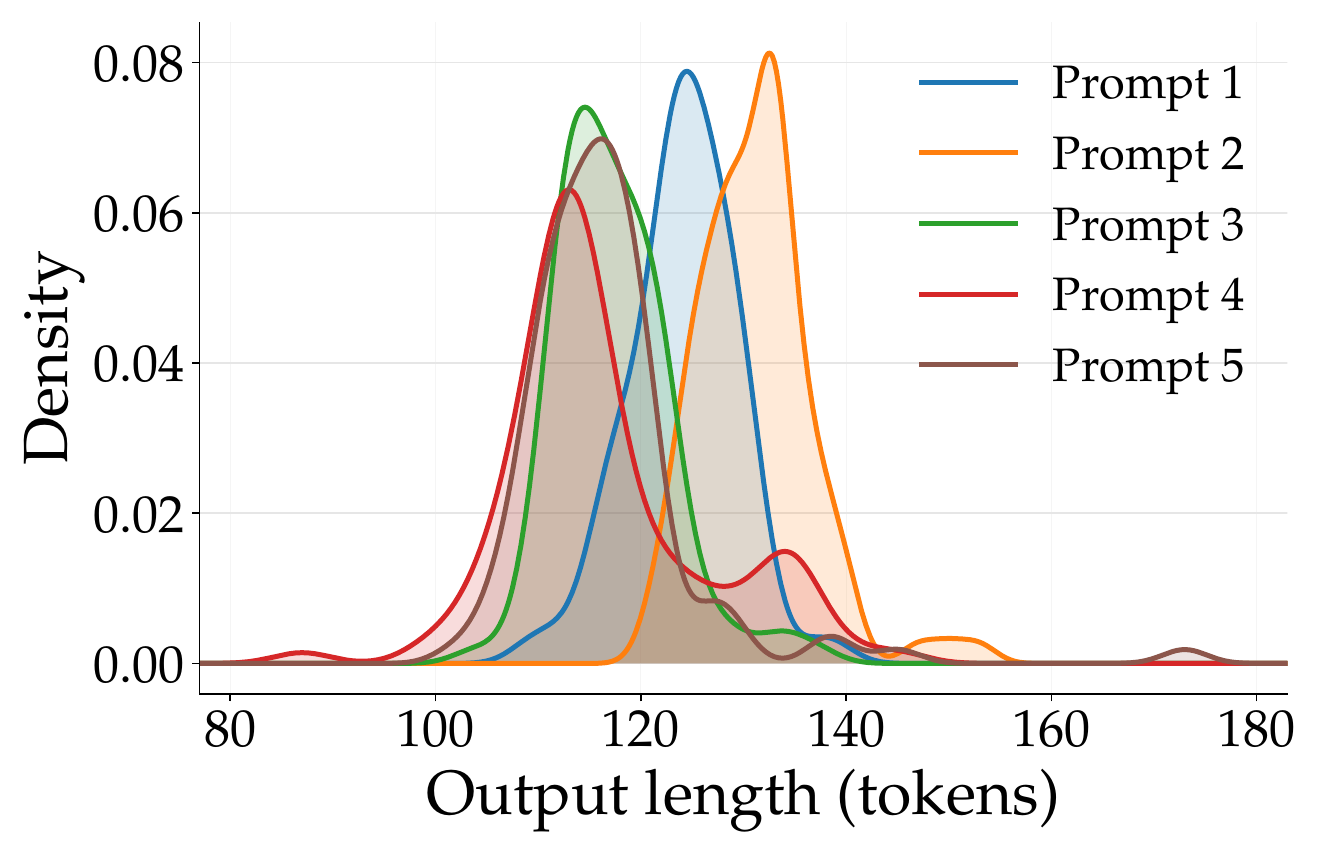}
        \caption{Math light5}
    \end{subfigure}
    \hfill
    \begin{subfigure}[t]{0.24\linewidth}
        \centering
        \includegraphics[width=\linewidth]{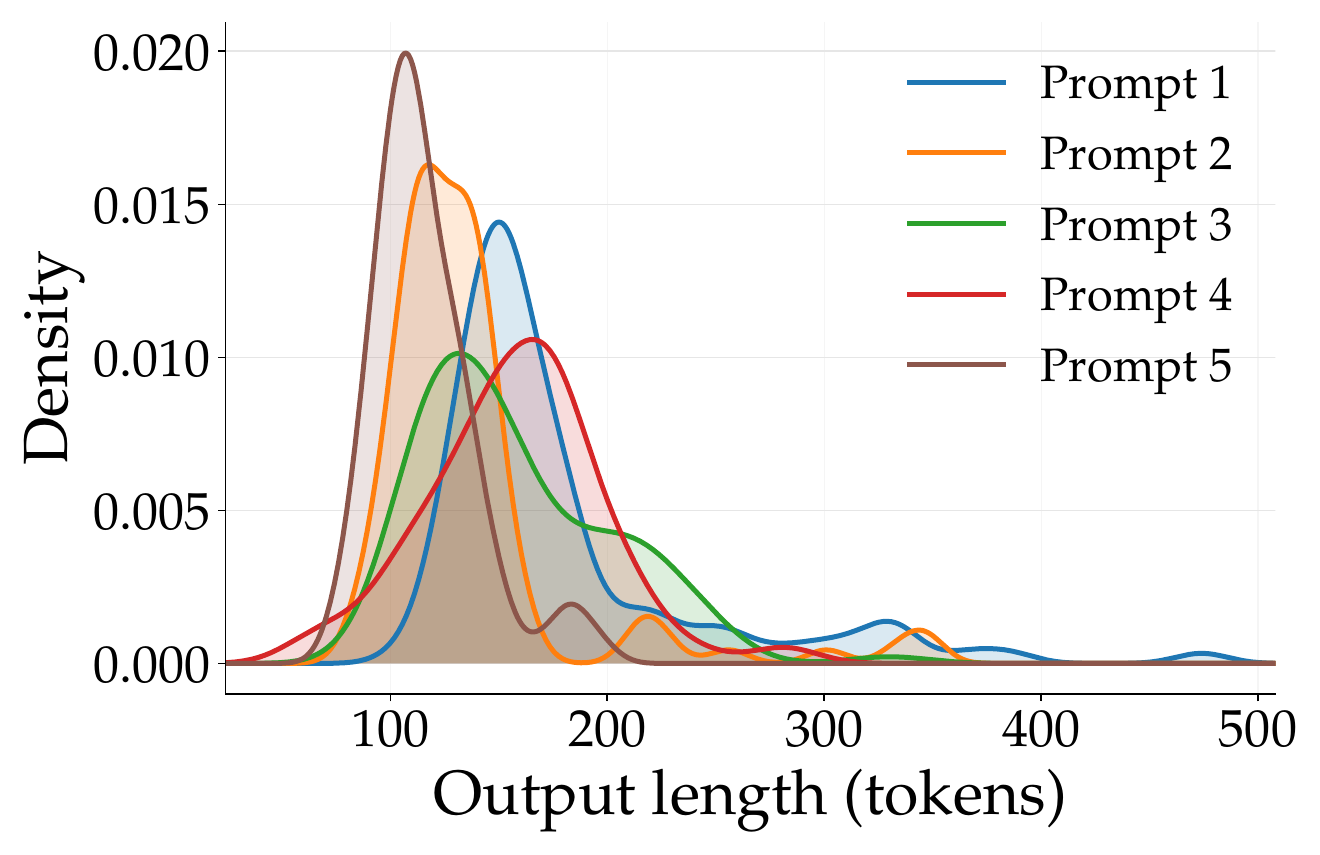}
        \caption{Math heavy5}
    \end{subfigure}
    \hfill
    \begin{subfigure}[t]{0.24\linewidth}
        \centering
        \includegraphics[width=\linewidth]{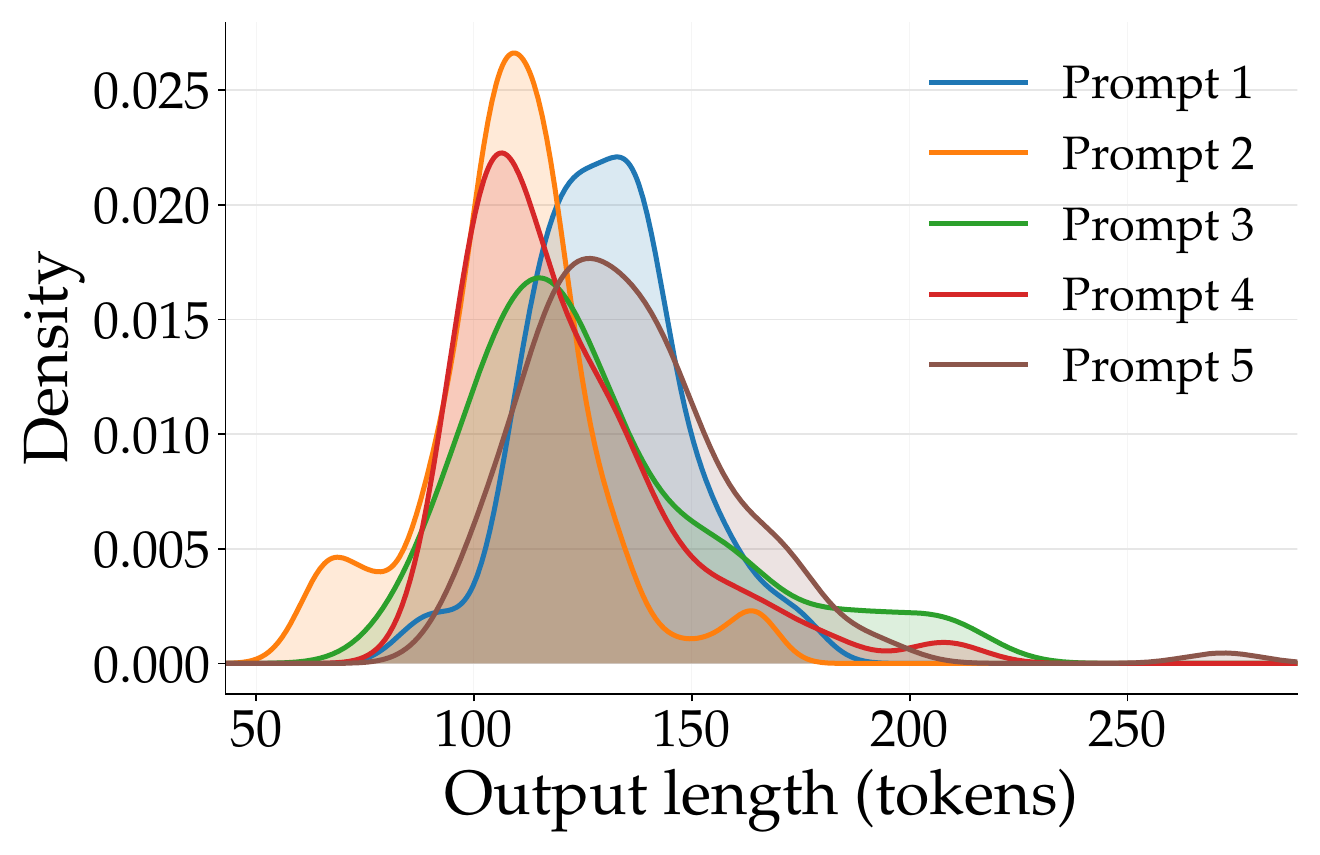}
        \caption{Coding light5}
    \end{subfigure}
    \hfill
    \begin{subfigure}[t]{0.24\linewidth}
        \centering
        \includegraphics[width=\linewidth]{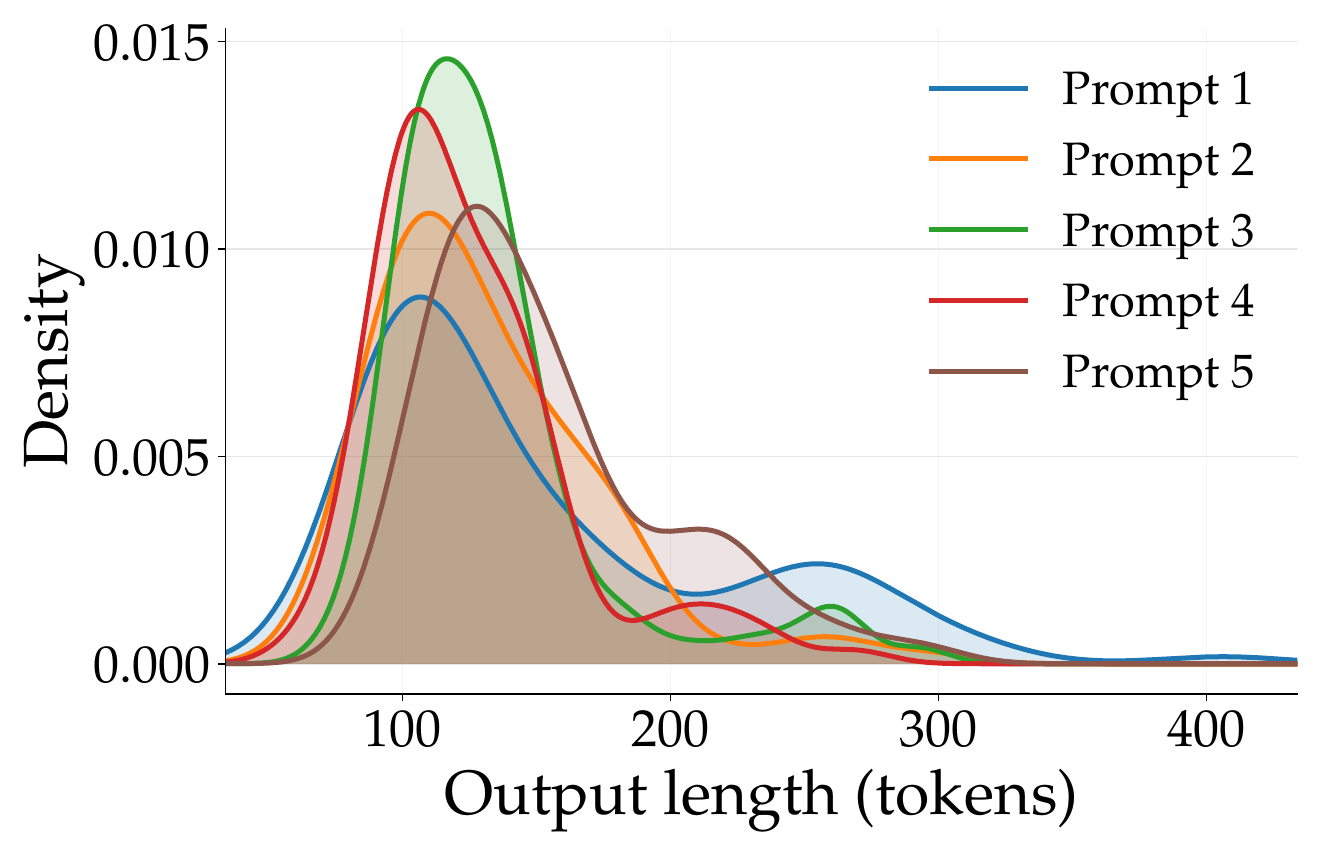}
        \caption{Coding heavy5}
    \end{subfigure}

    \vspace{0.5em}

    \begin{subfigure}[t]{0.24\linewidth}
        \centering
        \includegraphics[width=\linewidth]{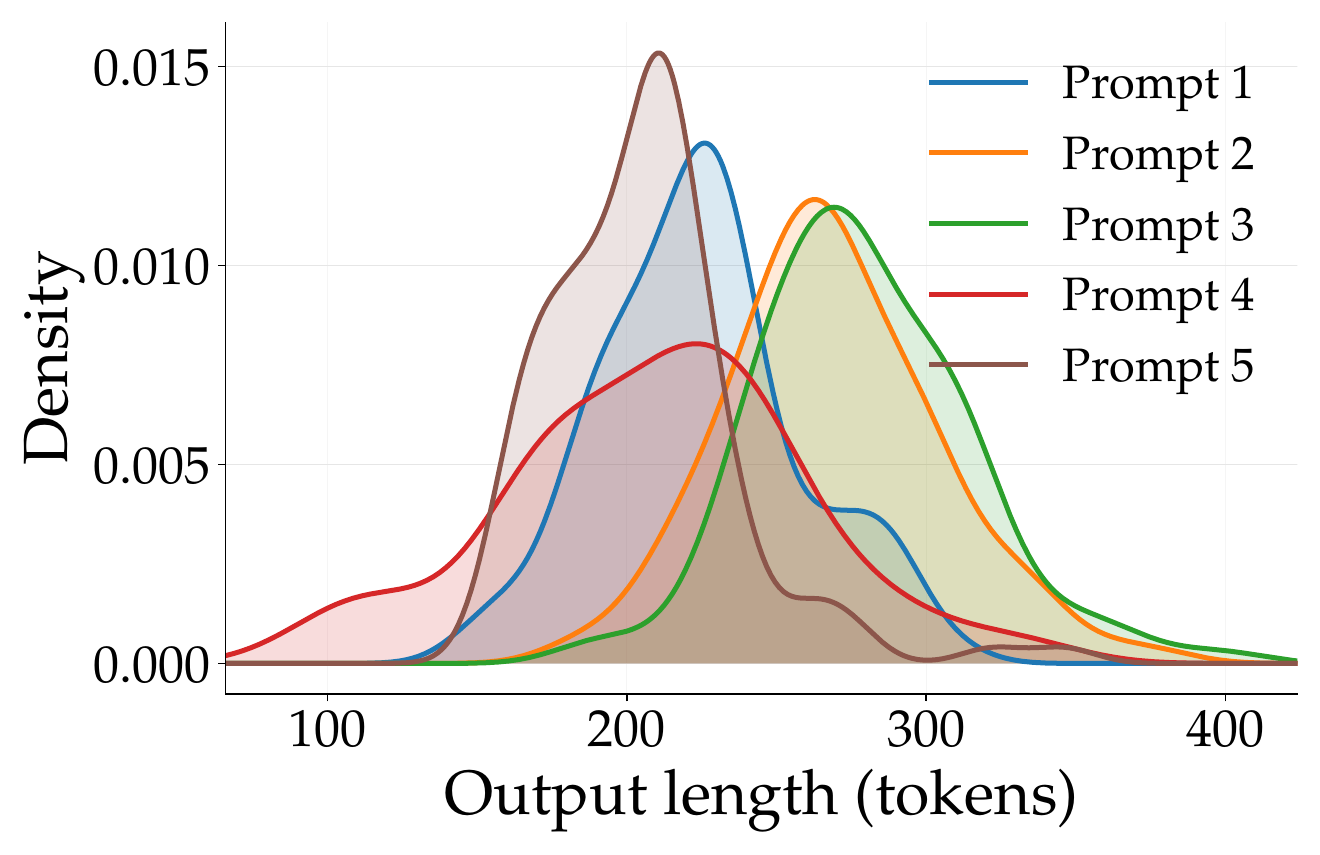}
        \caption{LongSequence light5}
    \end{subfigure}
    \hfill
    \begin{subfigure}[t]{0.24\linewidth}
        \centering
        \includegraphics[width=\linewidth]{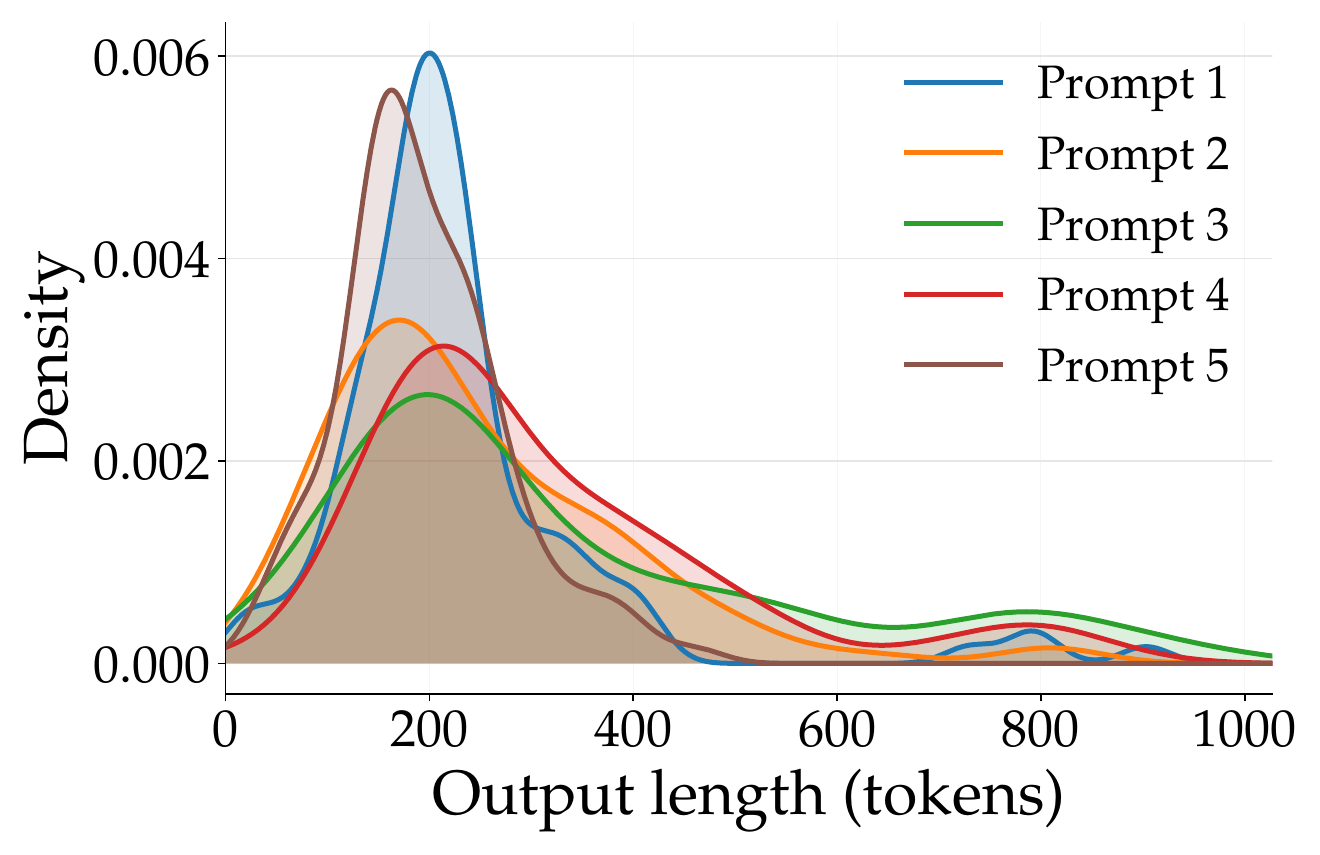}
        \caption{LongSequence heavy5}
    \end{subfigure}
    \hfill
    \begin{subfigure}[t]{0.24\linewidth}
        \centering
        \includegraphics[width=\linewidth]{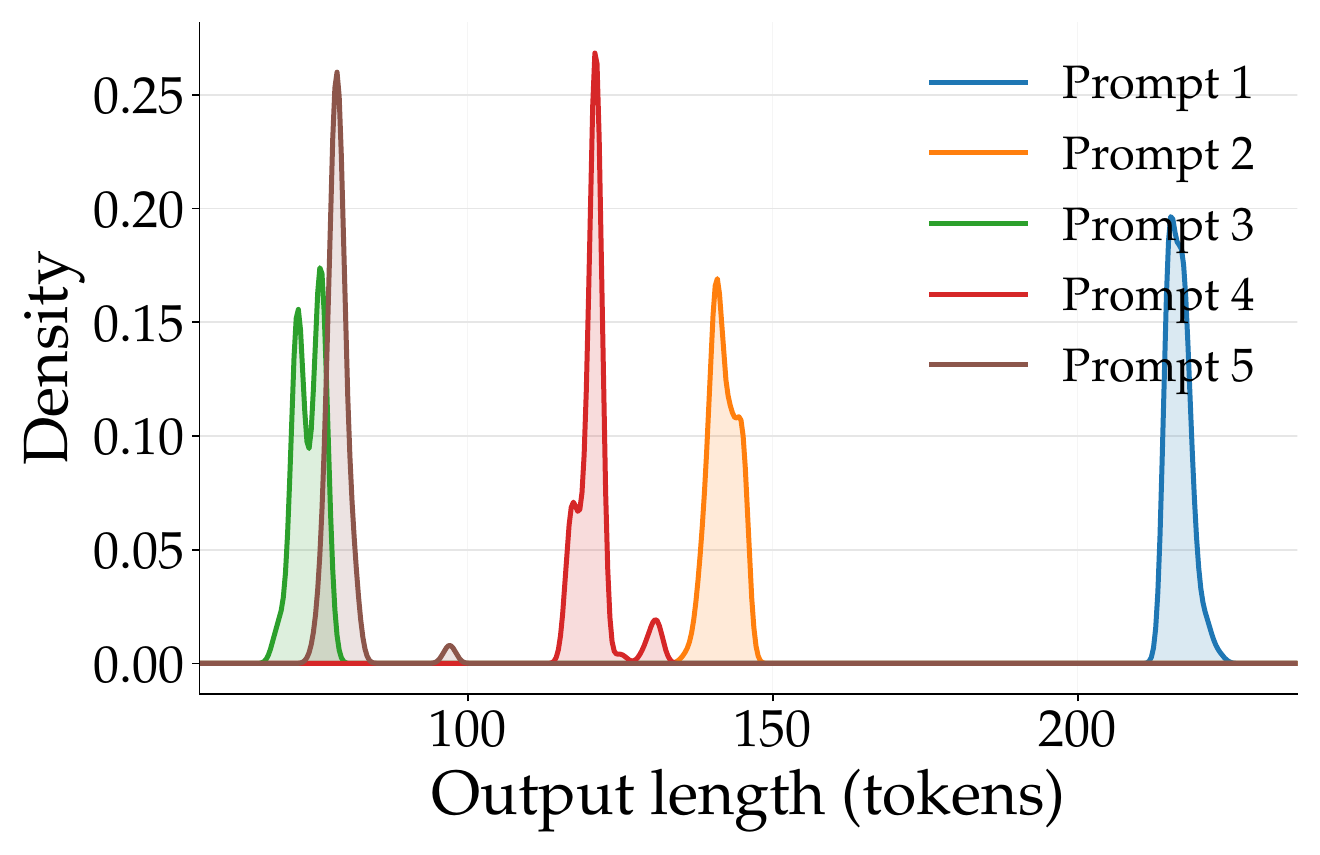}
        \caption{Chat light5}
    \end{subfigure}
    \hfill
    \begin{subfigure}[t]{0.24\linewidth}
        \centering
        \includegraphics[width=\linewidth]{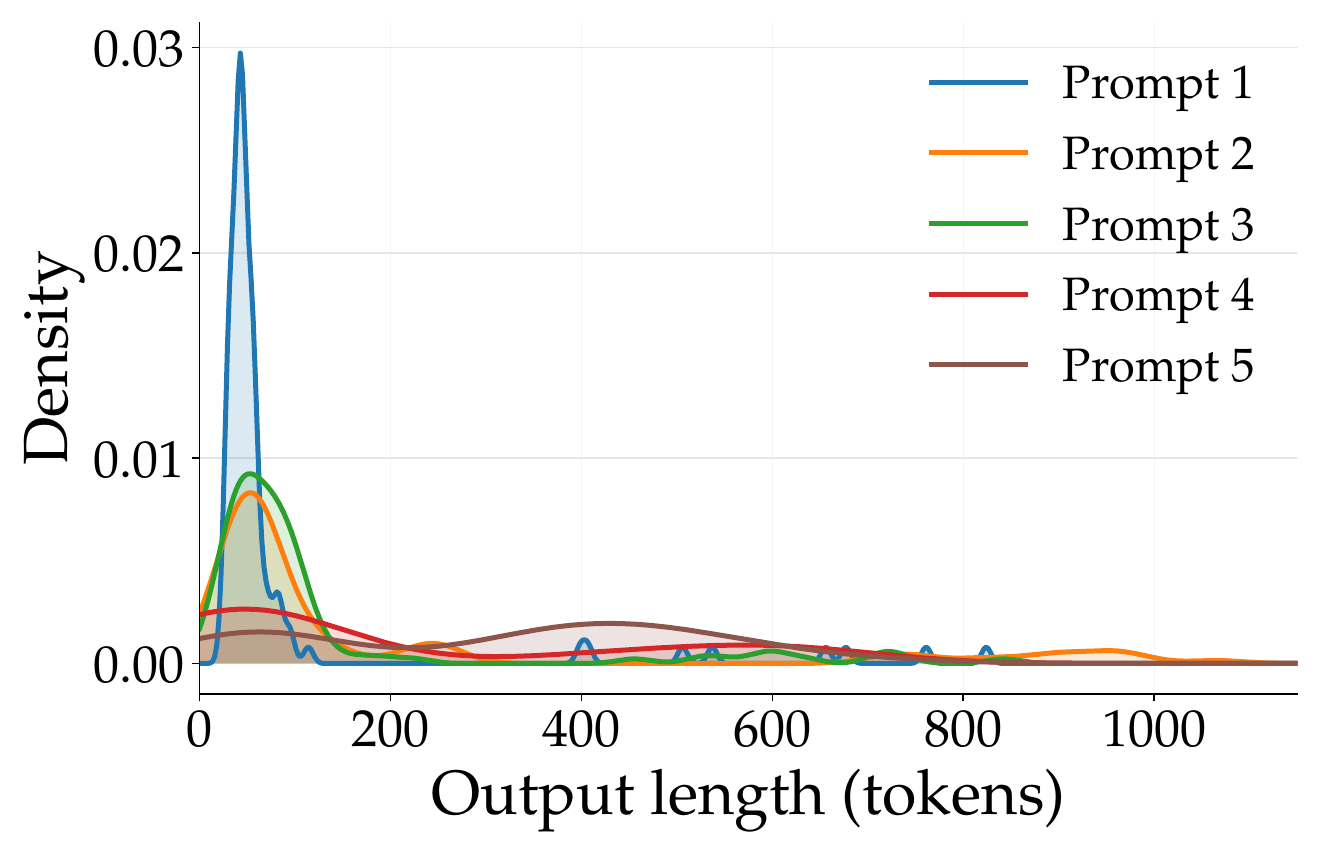}
        \caption{Chat heavy5}
    \end{subfigure}
    \caption{\textbf{Llama per-setting light/heavy overlays.} For each setting, \emph{light5} and \emph{heavy5} denote the five prompts with the smallest and largest $\max(\mathrm{length}) / \mathrm{median}(\mathrm{length})$ among the ten repeated-sampling prompts.}
    \label{fig:app_key_observations_llama_overlays}
\end{figure*}

\subsection{Detailed Key Observations Across Settings}
\label{sec:app_key_observations}

\paragraph{Setting.}
This section provides the detailed experimental results behind Figure~\ref{fig:key_observations}. We study four scenarios: Math = GSM8K, Coding = MBPP, LongSequence = LongBench, and Chat = LMSYS-Chat-1M. For the noise-floor summaries, GSM8K and MBPP pool their official train and test splits, whereas LongBench and LMSYS use the test side of the deterministic derived splits described above. Within each setting, we fix the served model, prompt format, and decoding setup, and then run $16$ repeated generations per prompt at temperature $0.8$ to compute prompt-level Median-MAE. For the heavy-tail diagnostics, we freeze a pool of $10$ representative train prompts per setting and run $100$ repeated generations for each selected prompt. GSM8K, MBPP, and LMSYS use the standard short-context serving path, whereas LongBench is reported under model-specific long-context budgets with explicit truncation metadata and reserved generation budget.

\paragraph{Detailed noise radius.}
Figure~\ref{fig:key_observations_noise_floor} in the main text provides the aggregated grouped boxplot, while Figure~\ref{fig:app_key_observations_waterfalls} shows how that variability is distributed across prompt percentiles. The setting-wise median prompt-level noise radius are $27.84$ and $16.13$ tokens for Qwen/Llama on Math, $21.66$ and $22.97$ on Coding, $42.94$ and $38.00$ on LongSequence, and $35.25$ and $33.38$ on Chat. The largest upper tails appear in the harder regimes: the setting-level $90$-th percentile reaches $71.4$ and $62.7$ tokens on Qwen/Llama LongSequence, and $135.9$ and $96.1$ on Qwen/Llama Chat. 

\paragraph{Heavy-tailed noise.}
Figures~\ref{fig:app_key_observations_qwen_overlays} and~\ref{fig:app_key_observations_llama_overlays} explain how we build the heavy-tail diagnostics behind the main-text examples. For each model--scenario pair, we first freeze $10$ representative prompts, then rank them by $\max(\mathrm{length}) / \mathrm{median}(\mathrm{length})$ under the $100$-repeat runs; \emph{light5} and \emph{heavy5} denote the bottom-$5$ and top-$5$ prompts under this ratio. Their max-to-median ratios are $2.91\times$, $2.42\times$, and $1.98\times$ for Qwen, and $2.99\times$, $3.27\times$, and $4.03\times$ for Llama. Across the appendix overlays, the heavy groups consistently show broader right tails than the light groups, which is why the main text treats these patterns as consistent with heavy-tailed behavior.

\section{Theoretical Analysis}
In this section, we provide the proof of Theorem \ref{theorem:estimation_error}.
\label{sec:app_theory}
\begin{proof}[Proof of Theorem \ref{theorem:estimation_error}]

\begin{equation*}
    \abs{\phi(x)^\top \theta_* - \phi(x)^\top \thetah_N}
    \le
    \norm{\phi(x)}_{V_{N}^{-1}}\norm{\theta_* - \thetah_N}_{V_{N}}.
\end{equation*}

Moreover, since the ridge estimator has a closed-form solution,
\begin{equation*}
    \thetah_N
    =
    V_N^{-1}
    \sbr{\sum_{i=1}^{N}\bar{L}_i\phi(x_i)},
    \qquad
    V_N \define \lambda I_d + \sum_{i=1}^{N}\phi(x_i)\phi(x_i)^\top.
\end{equation*}
The absolute prediction error can be bounded as follows:
\begin{align*}
    \thetah_N - \theta_*
    = {}& V_{N}^{-1}\sbr{\sum_{i=1}^{N}\bar{L}_i\phi(x_i)}- \theta_* \\
    = {}& V_{N}^{-1}\sbr{\sum_{i=1}^{N}\bar{L}_i\phi(x_i)}
    -V_{N}^{-1}\sbr{\lambda I_d + \sum_{i=1}^{N}\phi(x_i)\phi(x_i)^\top}\theta_* \\
    = {}& V_{N}^{-1}\sbr{\sum_{i=1}^{N}\sbr{\bar{L}_i - \phi(x_i)^\top\theta_*}\phi(x_i)-\lambda\theta_*}.
\end{align*}
Therefore, by the Cauchy--Schwarz inequality, for any $x \in \X$,
\begin{equation*}
\begin{split}
    \abs{\phi(x)^\top \theta_* - \phi(x)^\top \thetah_N}
    &\le
    \norm{\phi(x)}_{V_{N}^{-1}}\norm{\theta_* - \thetah_N}_{V_{N}} \\
    &\le
    \norm{\phi(x)}_{V_{N}^{-1}}
    \norm{\sum_{i=1}^{N}\sbr{\bar{L}_i - \phi(x_i)^\top\theta_*}\phi(x_i)-\lambda\theta_*}_{V_{N}^{-1}}.
\end{split}
\end{equation*}
Hence,
\begin{equation*}
\begin{split}
    \norm{\theta_* - \thetah_N}_{V_{N}}
    &\le
    \norm{\sum_{i=1}^{N}\sbr{\bar{L}_i - \phi(x_i)^\top\theta_*}\phi(x_i)-\lambda\theta_*}_{V_{N}^{-1}} \\
    &\le
    \norm{\sum_{i=1}^{N}\sbr{\bar{L}_i - \phi(x_i)^\top\theta_*}\phi(x_i)}_{V_{N}^{-1}}
    +
    \sqrt{\lambda}S.
\end{split}
\end{equation*}
We define $\bar{\eta}_i \define \bar{L}_i - \phi(x_i)^\top\theta_*$, and
\begin{equation*}
    S_i \define \sum_{j=1}^{i}\bar{\eta}_j\phi(x_j), \qquad S_0 \define 0.
\end{equation*}
For notational simplicity, let $\phi_i \define \phi(x_i)$. Since $V_i = V_{i-1} + \phi_i\phi_i^\top$, the Sherman--Morrison formula yields
\begin{equation*}
    V_i^{-1}
    =
    V_{i-1}^{-1}
    -
    \frac{V_{i-1}^{-1}\phi_i\phi_i^\top V_{i-1}^{-1}}
    {1+\phi_i^\top V_{i-1}^{-1}\phi_i}.
\end{equation*}
Define
\begin{equation*}
    a_i
    \define
    \frac{1}{1+\norm{\phi_i}_{V_{i-1}^{-1}}^2}
    =
    \frac{1}{1+\phi_i^\top V_{i-1}^{-1}\phi_i}.
\end{equation*}
Then, using $S_i = S_{i-1} + \bar{\eta}_i\phi_i$, we have
\begin{equation*}
\begin{split}
    \norm{S_i}_{V_i^{-1}}^2
    &=
    (S_{i-1}+\bar{\eta}_i\phi_i)^\top V_i^{-1}(S_{i-1}+\bar{\eta}_i\phi_i) \\
    &=
    \norm{S_{i-1}}_{V_{i-1}^{-1}}^2
    -
    a_i\sbr{\phi_i^\top V_{i-1}^{-1}S_{i-1}}^2
    +
    2a_i\bar{\eta}_i\phi_i^\top V_{i-1}^{-1}S_{i-1}
    +
    \norm{\phi_i}_{V_i^{-1}}^2\bar{\eta}_i^2.
\end{split}
\end{equation*}

Summing the above recursion from $i=1$ to $N$ and using $S_0 = 0$, we obtain
\begin{equation*}
\begin{split}
    \norm{S_{N}}_{V_{N}^{-1}}^2
    &=
    -\sum_{i=1}^{N} a_i\sbr{\phi(x_i)^\top V_{i-1}^{-1}S_{i-1}}^2
    +
    2\sum_{i=1}^{N} a_i\bar{\eta}_i\phi(x_i)^\top V_{i-1}^{-1}S_{i-1}
    +
    \sum_{i=1}^{N} \norm{\phi(x_i)}_{V_i^{-1}}^2\bar{\eta}_i^2.
\end{split}
\end{equation*}

Now define $\alpha_i \define a_i \phi(x_i)^\top V_{i-1}^{-1}S_{i-1}$ and
$\beta_i \define \norm{\phi(x_i)}_{V_{i-1}^{-1}}$, and let
$\alpha \define \sbr{\alpha_1,\dots,\alpha_{N}}$ and
$\beta \define \sbr{\beta_1,\dots,\beta_{N}}$.
Since
\begin{equation*}
    a_i\sbr{\phi(x_i)^\top V_{i-1}^{-1}S_{i-1}}^2
    =
    \frac{\alpha_i^2}{a_i}
    \ge \alpha_i^2,
\end{equation*}
and
\begin{equation*}
    \norm{\phi(x_i)}_{V_i^{-1}}^2
    =
    \frac{\beta_i^2}{1+\beta_i^2}
    \le \beta_i^2,
\end{equation*}
we further obtain
\begin{equation*}
\begin{split}
    \norm{S_{N}}_{V_{N}^{-1}}^2 \le
    -\sum_{i=1}^{N}\alpha_i^2
    +
    2\sum_{i=1}^{N}\alpha_i\bar{\eta}_i
    +
    \sum_{i=1}^{N}\beta_i^2\bar{\eta}_i^2  =
    -\norm{\alpha}_2^2
    +
    2\sum_{i=1}^{N}\alpha_i\bar{\eta}_i
    +
    \sum_{i=1}^{N}\beta_i^2\bar{\eta}_i^2.
\end{split}
\end{equation*}

Therefore, by Lemma~\ref{lemma:adapted_linear_explicit_R} and Lemma~\ref{lemma:adapted_quadratic_explicit_R}, with probability at least
\begin{equation*}
    1-\delta-4Ne^{-r/8},
\end{equation*}
the following two bounds hold simultaneously:
\begin{equation*}
    \sum_{i=1}^{N}\alpha_i\bar{\eta}_i
    \le
    \rho\norm{\alpha}_{1+\epsilon},
    \qquad
    \sum_{i=1}^{N}\beta_i^2\bar{\eta}_i^2
    \le
    C\rho\norm{\beta}_{1+\epsilon}^2,
\end{equation*}
where
\begin{equation*}
    C \define (4v)^{\frac{1}{1+\epsilon}},
    \qquad
    \rho \define 2C\ln\sbr{\frac{8N}{\delta}} + 4C^{-\epsilon}v.
\end{equation*}
Hence,
\begin{equation*}
    \norm{S_{N}}_{V_{N}^{-1}}^2
    \le
    -\norm{\alpha}_2^2
    +
    2\rho\norm{\alpha}_{1+\epsilon}
    +
    C\rho\norm{\beta}_{1+\epsilon}^2.
\end{equation*}

By H\"older's inequality,
\begin{equation*}
    \norm{\alpha}_{1+\epsilon}
    \le
    N^{\frac{1-\epsilon}{2(1+\epsilon)}}\norm{\alpha}_2,
    \qquad
    \norm{\beta}_{1+\epsilon}^2
    \le
    N^{\frac{1-\epsilon}{1+\epsilon}}\norm{\beta}_2^2.
\end{equation*}
Thus,
\begin{equation*}
\begin{split}
    \norm{S_{N}}_{V_{N}^{-1}}^2
    &\le
    -\norm{\alpha}_2^2
    +
    2\rho N^{\frac{1-\epsilon}{2(1+\epsilon)}}\norm{\alpha}_2
    +
    C\rho N^{\frac{1-\epsilon}{1+\epsilon}}\norm{\beta}_2^2.
\end{split}
\end{equation*}
Using $-u^2 + 2au \le a^2$ with
\begin{equation*}
    u = \norm{\alpha}_2,
    \qquad
    a = \rho N^{\frac{1-\epsilon}{2(1+\epsilon)}},
\end{equation*}
we get
\begin{equation*}
    \norm{S_{N}}_{V_{N}^{-1}}^2
    \le
    \rho^2 N^{\frac{1-\epsilon}{1+\epsilon}}
    +
    C\rho N^{\frac{1-\epsilon}{1+\epsilon}}\norm{\beta}_2^2.
\end{equation*}
Applying Lemma~\ref{lemma:potential_lemma},
\begin{equation*}
    \norm{\beta}_2^2
    =
    \sum_{i=1}^{N}\beta_i^2
    =
    \sum_{i=1}^{N}\norm{\phi(x_i)}_{V_{i-1}^{-1}}^2
    \le
    2d\log\sbr{1+\frac{N}{\lambda d}}.
\end{equation*}
Therefore, with probability at least $1-\delta-4Ne^{-r/8}$,
we have
\begin{equation*}
    \norm{S_{N}}_{V_{N}^{-1}}^2
    \le
    \rho^2 N^{\frac{1-\epsilon}{1+\epsilon}}
    +
    2C\rho d N^{\frac{1-\epsilon}{1+\epsilon}}
    \log\sbr{1+\frac{N}{\lambda d}}.
\end{equation*}
Combining this bound with
\begin{equation*}
    \norm{\theta_* - \thetah_N}_{V_{N}}
    \le
    \norm{S_{N}}_{V_{N}^{-1}}+\sqrt{\lambda}S,
\end{equation*}
with probability at least $1-\delta-4Ne^{-r/8}$, we have
\begin{equation*}
    \norm{\theta_* - \thetah_N}_{V_{N}}
    \le
    \sqrt{\lambda}S
    +
    \sqrt{
        \rho^2 N^{\frac{1-\epsilon}{1+\epsilon}}
        +
        2C\rho d N^{\frac{1-\epsilon}{1+\epsilon}}
        \log\sbr{1+\frac{N}{\lambda d}}
    }.
\end{equation*}
Consequently, for any $x \in \X$, with probability at least $1-\delta-4Ne^{-r/8}$,
\begin{equation*}
\begin{split}
    \abs{\phi(x)^\top \theta_* - \phi(x)^\top \thetah_N}
    &\le
    \norm{\phi(x)}_{V_{N}^{-1}}
    \sbr{
        \sqrt{\lambda}S
        +
        \sqrt{
            \rho^2 N^{\frac{1-\epsilon}{1+\epsilon}}
            +
            2C\rho d N^{\frac{1-\epsilon}{1+\epsilon}}
            \log\sbr{1+\frac{N}{\lambda d}}
        }
    }.
\end{split}
\end{equation*}

\end{proof}

\subsection{Useful Lemmas}

\begin{myLemma}[Adapted linear concentration with explicit repeated-sampling probability]
    \label{lemma:adapted_linear_explicit_R}
    Suppose Assumption~\ref{assumption:noise_symmetric} holds. For any fixed dataset size $N$ and any fixed coefficients $\alpha_1,\dots,\alpha_{N} \in \R$, let
    \begin{equation*}
        \alpha \define \sbr{\alpha_1,\dots,\alpha_{N}},
        \qquad
        C \define (4v)^{\frac{1}{1+\epsilon}},
        \qquad
        \rho \define 2C\ln\sbr{\frac{8N}{\delta}} + 4C^{-\epsilon}v.
    \end{equation*}
    Then,
    \begin{equation*}
        \sum_{i=1}^{N}\alpha_i\bar{\eta}_i
        \le
        \rho\norm{\alpha}_{1+\epsilon}
    \end{equation*}
    holds with probability at least
    \begin{equation*}
        1-\frac{\delta}{2}-2Ne^{-r/8}.
    \end{equation*}
    \end{myLemma}
    
    \begin{proof}
    The proof follows the same decomposition as Lemma 10 in CRMM~\citep{NeurIPS'23:hvtglb}, with two modifications.
    
    First, in the proof of CRMM Lemma 10, the quantity $r$ enters through the control of the bad event
    \begin{equation*}
        \abs{\alpha_i\bar{\eta}_i} > C\norm{\alpha}_{1+\epsilon}.
    \end{equation*}
    For each fixed $i$, write the $r$ repeated noises as $\eta_{i,1},\dots,\eta_{i,r}$, and recall that
    \begin{equation*}
        \bar{\eta}_i = \operatorname{median}\mbr{\eta_{i,1},\dots,\eta_{i,r}}.
    \end{equation*}
    By Markov's inequality and $\E\mbr{\abs{\eta_{i,j}}^{1+\epsilon}\mid x_i}\le v$, we have
    \begin{equation*}
        \mathbb{P}\sbr{\abs{\alpha_i\eta_{i,j}} > C\norm{\alpha}_{1+\epsilon}}
        \le
        \frac{\abs{\alpha_i}^{1+\epsilon}v}{C^{1+\epsilon}\norm{\alpha}_{1+\epsilon}^{1+\epsilon}}
        \le
        \frac{1}{4},
    \end{equation*}
    where the last step uses $C=(4v)^{\frac{1}{1+\epsilon}}$.
    
    Now define
    \begin{equation*}
        B_{i,j}^+ \define \ind{\alpha_i\eta_{i,j} > C\norm{\alpha}_{1+\epsilon}},
        \qquad
        B_{i,j}^- \define \ind{\alpha_i\eta_{i,j} < -C\norm{\alpha}_{1+\epsilon}}.
    \end{equation*}
    If $\alpha_i\bar{\eta}_i > C\norm{\alpha}_{1+\epsilon}$, then at least half of the repeated samples satisfy
    \begin{equation*}
        \alpha_i\eta_{i,j} > C\norm{\alpha}_{1+\epsilon},
    \end{equation*}
    that is,
    \begin{equation*}
        \sum_{j=1}^r B_{i,j}^+ \ge \frac{r}{2}.
    \end{equation*}
    Since $\E[B_{i,j}^+] \le \frac14$, Hoeffding's inequality gives
    \begin{equation*}
        \mathbb{P}\sbr{\alpha_i\bar{\eta}_i > C\norm{\alpha}_{1+\epsilon}}
        \le
        \exp\!\sbr{-\frac{r}{8}}.
    \end{equation*}
    By the same argument,
    \begin{equation*}
        \mathbb{P}\sbr{\alpha_i\bar{\eta}_i < -C\norm{\alpha}_{1+\epsilon}}
        \le
        \exp\!\sbr{-\frac{r}{8}}.
    \end{equation*}
    Therefore,
    \begin{equation*}
        \mathbb{P}\sbr{\abs{\alpha_i\bar{\eta}_i} > C\norm{\alpha}_{1+\epsilon}}
        \le
        2e^{-r/8}.
    \end{equation*}
    Summing over $i=1,\dots,N$, we obtain
    \begin{equation*}
        \sum_{i=1}^{N}\mathbb{P}\sbr{\abs{\alpha_i\bar{\eta}_i} > C\norm{\alpha}_{1+\epsilon}}
        \le
        2Ne^{-r/8}.
    \end{equation*}
    
    Second, in the truncated part, the original proof of CRMM Lemma 10 invokes its Lemma 8 to claim that the median term has $(1+\epsilon)$-th moment bounded by $rv$. Here we replace that step by Lemma~\ref{lemma:median_moment}, which yields the sharper bound
    \begin{equation*}
        \E\mbr{\abs{\bar{\eta}_i}^{1+\epsilon}\mid x_i} \le 2v.
    \end{equation*}
    Moreover, by Assumption~\ref{assumption:noise_symmetric}, the conditional distribution of $\bar{\eta}_i$ is symmetric around zero, hence
    \begin{equation*}
        \E\mbr{\bar{\eta}_i \mid \mathcal{F}_{i-1}} = 0.
    \end{equation*}
    Therefore, Lemma 9 in CRMM can be applied to the truncated sum with
    \begin{equation*}
        v_1 = 2v,
        \qquad
        \delta' = \frac{\delta}{4N}.
    \end{equation*}
    This gives
    \begin{equation*}
        \sum_{i=1}^{N}\alpha_i\bar{\eta}_i \ind{\abs{\alpha_i\bar{\eta}_i}\le C\norm{\alpha}_{1+\epsilon}}
        \le
        \rho \norm{\alpha}_{1+\epsilon}
    \end{equation*}
    with probability at least $1-\frac{\delta}{2}$, where
    \begin{equation*}
        \rho = 2C\ln\sbr{\frac{8N}{\delta}} + 4C^{-\epsilon}v.
    \end{equation*}
    
    Combining the truncated bound with the above control of the bad event finishes the proof. All remaining steps are identical to the proof of Lemma 10 in CRMM.
    \end{proof}

    \begin{myLemma}[Adapted quadratic concentration with explicit repeated-sampling probability]
    \label{lemma:adapted_quadratic_explicit_R}
    Suppose Assumptions~\ref{assumption:noise_symmetric} holds. For any fixed dataset size $N$ and any fixed coefficients $\beta_1,\dots,\beta_{N} \ge 0$, let
    \begin{equation*}
        \beta \define \sbr{\beta_1,\dots,\beta_{N}},
        \qquad
        C \define (4v)^{\frac{1}{1+\epsilon}},
        \qquad
        \rho \define 2C\ln\sbr{\frac{8N}{\delta}} + 4C^{-\epsilon}v.
    \end{equation*}
    Then,
    \begin{equation*}
        \sum_{i=1}^{N}\beta_i^2\bar{\eta}_i^2
        \le
        C\rho\norm{\beta}_{1+\epsilon}^2
    \end{equation*}
    holds with probability at least
    \begin{equation*}
        1-\frac{\delta}{2}-2Ne^{-r/8}.
    \end{equation*}
    \end{myLemma}
    
    \begin{proof}
    The proof follows Lemma 11 in CRMM, again with the same two modifications as above.
    
    We first control the bad event
    \begin{equation*}
        \abs{\beta_i\bar{\eta}_i} > C\norm{\beta}_{1+\epsilon}.
    \end{equation*}
    For each repeated sample $\eta_{i,j}$, Markov's inequality yields
    \begin{equation*}
        \mathbb{P}\sbr{\abs{\beta_i\eta_{i,j}} > C\norm{\beta}_{1+\epsilon}}
        \le
        \frac{\beta_i^{1+\epsilon}v}{C^{1+\epsilon}\norm{\beta}_{1+\epsilon}^{1+\epsilon}}
        \le
        \frac14,
    \end{equation*}
    where $C=(4v)^{\frac{1}{1+\epsilon}}$.
    As in the proof of Lemma~\ref{lemma:adapted_linear_explicit_R}, if
    \begin{equation*}
        \beta_i\bar{\eta}_i > C\norm{\beta}_{1+\epsilon},
    \end{equation*}
    then more than half of the repeated samples must satisfy
    \begin{equation*}
        \beta_i\eta_{i,j} > C\norm{\beta}_{1+\epsilon},
    \end{equation*}
    and Hoeffding's inequality gives
    \begin{equation*}
        \mathbb{P}\sbr{\beta_i\bar{\eta}_i > C\norm{\beta}_{1+\epsilon}}
        \le
        e^{-r/8}.
    \end{equation*}
    The same bound holds for the lower tail, so
    \begin{equation*}
        \mathbb{P}\sbr{\abs{\beta_i\bar{\eta}_i} > C\norm{\beta}_{1+\epsilon}}
        \le
        2e^{-r/8}.
    \end{equation*}
    Summing over $i=1,\dots,N$ yields
    \begin{equation*}
        \sum_{i=1}^{N}\mathbb{P}\sbr{\abs{\beta_i\bar{\eta}_i} > C\norm{\beta}_{1+\epsilon}}
        \le
        2Ne^{-r/8}.
    \end{equation*}
    
    For the truncated part, the proof of CRMM Lemma 11 invokes its auxiliary quadratic truncation lemma (their Lemma 12), and the only place where the median moment enters is through the bound on
    \begin{equation*}
        \E\mbr{\abs{\bar{\eta}_i}^{1+\epsilon}\mid x_i}.
    \end{equation*}
    Using Lemma~\ref{lemma:median_moment}, we replace the original bound $rv$ by the sharper estimate
    \begin{equation*}
        \E\mbr{\abs{\bar{\eta}_i}^{1+\epsilon}\mid x_i} \le 2v.
    \end{equation*}
    Therefore, Lemma 12 of CRMM~\citep{NeurIPS'23:hvtglb} applies with
    \begin{equation*}
        v_1 = 2v,
        \qquad
        \delta' = \frac{\delta}{4N}.
    \end{equation*}
    This yields
    \begin{equation*}
        \sum_{i=1}^{N}\beta_i^2\bar{\eta}_i^2
        \ind{\beta_i^2\bar{\eta}_i^2 \le C^2\norm{\beta}_{1+\epsilon}^2}
        \le
        C\rho \norm{\beta}_{1+\epsilon}^2
    \end{equation*}
    with probability at least $1-\frac{\delta}{2}$, where
    \begin{equation*}
        \rho = 2C\ln\sbr{\frac{8N}{\delta}} + 4C^{-\epsilon}v.
    \end{equation*}
    
    Combining the truncated part with the above control of the bad event proves the claim. Apart from replacing the original median moment bound $rv$ by the sharper estimate $2v$, and keeping the repeated-sampling probability explicitly as $e^{-r/8}$ instead of substituting a concrete value of $r$, all other steps follow exactly the proof of Lemma 11 in CRMM.
    \end{proof}

\begin{myLemma}[Moment bound for the repeated-sampling median]
    \label{lemma:median_moment}
    Let $X_1,\dots,X_r$ be independent copies of a real-valued random variable $X$ such that $\E\mbr{\abs{X}^{1+\epsilon}} \le v.$ Define $\bar{X} \define \operatorname{median}\mbr{X_1,\dots,X_r}$. Then we have $\E\mbr{\abs{\bar{X}}^{1+\epsilon}} \le 2v$.
\end{myLemma}

\begin{proof}
    Let $k \define \ceil{\frac{r}{2}}$. For any $u \ge 0$, define
    \begin{equation*}
        N_u^+ \define \sum_{j=1}^r \ind{X_j > u}, \qquad
        N_u^- \define \sum_{j=1}^r \ind{X_j < -u}.
    \end{equation*}
    If $\bar{X} > u$, then at least $k$ samples are strictly larger than $u$, and hence
    \begin{equation*}
        \bbr{\bar{X} > u} \subseteq \bbr{N_u^+ \ge k}.
    \end{equation*}
    Similarly, if $\bar{X} < -u$, then at least $k$ samples are strictly smaller than $-u$, and hence
    \begin{equation*}
        \bbr{\bar{X} < -u} \subseteq \bbr{N_u^- \ge k}.
    \end{equation*}
    Therefore, by Markov's inequality,
    \begin{equation*}
        \mathbb{P}\sbr{\bar{X} > u} \le \mathbb{P}\sbr{N_u^+ \ge k} \le \frac{\E\sbr{N_u^+}}{k} = \frac{1}{k} \sum_{j=1}^r \mathbb{P}\sbr{X_j > u} = \frac{r}{k}\,\mathbb{P}\sbr{X_1 > u}.
    \end{equation*}
    where the last step uses the i.i.d.\ assumption. By the same argument,
    \begin{equation*}
        \mathbb{P}\sbr{\bar{X} < -u}
        \le \frac{r}{k}\,\mathbb{P}\sbr{X_1 < -u}.
    \end{equation*}
    Summing the two bounds, we obtain
    \begin{equation*}
        \mathbb{P}\sbr{\abs{\bar{X}} > u} \le \mathbb{P}\sbr{\bar{X} > u} + \mathbb{P}\sbr{\bar{X} < -u} \le \frac{r}{k}\,\mathbb{P}\sbr{\abs{X_1} > u}.
    \end{equation*}

    Now apply the tail-integral identity:
    \begin{equation*}
        \begin{split}
            \E\mbr{\abs{\bar{X}}^{1+\epsilon}} &= (1+\epsilon)\int_0^\infty u^\epsilon \mathbb{P}\sbr{\abs{\bar{X}} > u} \diff u\le \frac{r}{k}(1+\epsilon)\int_0^\infty u^\epsilon \mathbb{P}\sbr{\abs{X_1} > u} \diff u \\
            &= \frac{r}{k} \E\mbr{\abs{X_1}^{1+\epsilon}} \le \frac{r}{k}v \le 2v.
        \end{split}
    \end{equation*}
    This completes the proof.
\end{proof}

\begin{myLemma}[Potential lemma]
    \label{lemma:potential_lemma}
    For
    \begin{equation*}
        V_N = \lambda I_d + \sum_{i=1}^N \phi(x_i)\phi(x_i)^\top,
    \end{equation*}
    we have
    \begin{equation*}
        \sum_{i=1}^{N}\norm{\phi(x_i)}_{V_{i-1}^{-1}}^2
        \le
        2\log \frac{\det(V_{N})}{\det(\lambda I_d)}
        \le
        2d\log\sbr{1+\frac{N}{\lambda d}}.
    \end{equation*}
    \end{myLemma}

\end{document}